\documentclass[conference,compsoc]{IEEEtran}

\ifCLASSOPTIONcompsoc
  \usepackage[nocompress]{cite}
\else
  \usepackage{cite}
\fi
\ifCLASSINFOpdf
\else
\fi
\usepackage{bm}
\usepackage{amsmath,amssymb,amsfonts}
\usepackage{bbm}
\usepackage{mathtools}
\usepackage{upgreek}
\usepackage{cases}
\usepackage{xintexpr}
\usepackage{pifont}
\usepackage{ntheorem}
\usepackage{booktabs}
\usepackage{multirow}
\usepackage{graphicx}
\usepackage{subcaption}
\usepackage{overpic}
\usepackage{textcomp}
\usepackage{xcolor}
\usepackage{xifthen}
\usepackage{multicol}
\usepackage{cuted}
\usepackage{cite}
\usepackage{enumitem}
\usepackage[switch]{lineno}
\usepackage[numbers,sort&compress]{natbib}
\usepackage[ruled,linesnumbered,noend]{algorithm2e}
\usepackage{hyperref,soul}
\usepackage{cancel}
\usepackage{nicematrix}
\usepackage{afterpage}
\DeclareMathAlphabet{\mathpzc}{OT1}{pzc}{m}{it}

\def\BibTeX{{\rm B\kern-.05em{\sc i\kern-.025em b}\kern-.08em
    N\kern-.1667em\lower.7ex\hbox{E}\kern-.125emX}}

\newcommand{\highlight}[1]{\vspace{1mm}\noindent{}\textbf{#1}}

\newcommand{\w}[1][]{
\ifthenelse{\isempty{#1}}
{\mathbf{w}}
{\mathbf{w}^{(#1)}}}
\newcommand{\tw}[1][]{
\ifthenelse{\isempty{#1}}
{\widetilde{\mathbf{w}}}
{\widetilde{\mathbf{w}}^{(#1)}}}
\newcommand{\hw}[1][]{
\ifthenelse{\isempty{#1}}
{\widehat{\mathbf{w}}}
{\widehat{\mathbf{w}}^{(#1)}}}

\newcommand{\X}{\mathbf{X}}
\newcommand{\hX}{\hat{\mathbf{X}}}

\newcommand{\x}{\mathbf{x}}
\newcommand{\hx}{\hat{\mathbf{x}}}

\newcommand{\bx}{\breve{\mathbf{x}}}

\newcommand{\z}{\mathbf{z}}

\newcommand{\set}{\mathcal{S}}

\newcommand{\recon}{\boldsymbol{\psi}}
\newcommand{\Recon}{\boldsymbol{\Psi}}

\newcommand{\update}[1][k]{\boldsymbol{g}^{(#1)}}
\newcommand{\tupdate}[1][k]{\widetilde{\boldsymbol{g}}^{(#1)}}

\newcommand{\mG}{\mathbf{G}}

\newcommand{\prob}{\mathbb{P}}

\newcommand{\argmin}{\mathop{\mathrm{argmin}}\limits} 

\newcommand{\expect}[1][]{
\ifthenelse{\isempty{#1}}
{\mathbb{E}}
{\mathbb{E}\left[#1\right]}}

\newcommand{\fw}[1][]{
\ifthenelse{\isempty{#1}}
{\mathbf{\bar{w}}}
{\mathbf{\bar{w}}^{(#1)}}
}

\newcommand{\normsq}[1]
{\big\| #1\big\|_2^2}

\newcommand{\noise}{\mathbf{n}}

\newcommand{\pdist}[1][m]{\mathcal{P}_{m}}

\newcommand{\p}[1][]{
\ifthenelse{\isempty{#1}}
{\boldsymbol{p}}
{\boldsymbol{p}^{(#1)}}
}

\newcommand{\cirone}
{\text{\ding{172}}}
\newcommand{\cirtwo}
{\text{\ding{173}}}
\newcommand{\cirthree}
{\text{\ding{174}}}

\newcommand{\intab}[2][0.75]{
\scalebox{#1}{\textrm{#2}}
}

\newcommand\addpicture[3]{%
\setbox\mybox=\hbox{\includegraphics[scale=#3]{#2}}
\myboxwidth\wd\mybox    
\renewcommand\windowpagestuff{%
\includegraphics[scale=#3]{#2}
\captionof{figure}{A test figure.}}
\parpic[#1]{%
\begin{minipage}{\myboxwidth}
 \windowpagestuff 
\end{minipage} 
} }

\newenvironment{proof}[1][]{
\ifthenelse{\isempty{#1}}
{\par\vspace*{-1mm}\noindent\textit{Proof.} }
{\par\vspace*{-2mm}\noindent\textit{Proof of #1.} }}
{\hfill$\square$ \vspace*{2mm}}

\theoremstyle{mystyle} 
\newtheorem{Theorem}{Theorem}

\newtheorem{Remark}{Remark}
\newtheorem{Definition}{Definition}
\newtheorem{Example}{Example}

\usepackage{tikz}

\newcommand{\enc}{E}
\newcommand{\dec}{D}

\usepackage{tcolorbox}

\definecolor{fedvote}{RGB}{235,255,251}
\definecolor{byzantinefedvote}{RGB}{235,246,255}

\newtheorem{Assumption}{Assumption}

\definecolor{mygreen}{RGB}{46, 125, 50}
\definecolor{myorange}{RGB}{230, 74, 25}

\ifodd 1

\newcommand{\com}[1]{\textbf{\color{blue}([KY]: #1)}}
\newcommand{\comRJ}[1]{\textbf{\color{red}([RJ]: #1)}}
\newcommand{\CommentWong}[1]{\textcolor[rgb]{1,0,0}{[Wong: #1]}}
\newcommand{\comHD}[1]{\textbf{\color{red}([HD]: #1)}}
\else

\newcommand{\com}[1]{}
\newcommand{\comRJ}[1]{}
\newcommand{\CommentWong}[1]{}
\newcommand{\comHD}[1]{}
\fi

\begin{document}
%
\title{Byzantine Outside, Curious Inside: \\ Reconstructing Data Through Malicious Updates}

\newcommand*{\affaddr}[1]{#1} 
\newcommand*{\affmark}[1][*]{\textsuperscript{#1}}
\newcommand*{\email}[1]{\texttt{#1}}

\author{%
Kai Yue,\affmark[1] Richeng Jin,\affmark[2] Chau-Wai Wong,\affmark[1] and Huaiyu Dai\affmark[1]\\
\affaddr{\affmark[1]North Carolina State University; \{kyue, chauwai.wong, hdai\}@ncsu.edu } \\
\affaddr{\affmark[2]Zhejiang University; richengjin@zju.edu.cn}\\
}

\maketitle


\begin{abstract}

Federated learning~(FL) enables decentralized machine learning without sharing raw data, allowing multiple clients to collaboratively learn a global model. 
However, studies reveal that privacy leakage is possible under commonly adopted FL protocols. 
In particular, a server with access to client gradients can synthesize data resembling the clients' training data. 
In this paper, we introduce a novel threat model in FL, named the maliciously curious client, where a client manipulates its own gradients with the goal of inferring private data from peers. 
This attacker uniquely exploits the strength of a Byzantine adversary, traditionally aimed at undermining model robustness, and repurposes it to facilitate data reconstruction attack.  
We begin by formally defining this novel client-side threat model and providing a theoretical analysis that demonstrates its ability to achieve significant reconstruction success during FL training. 
To demonstrate its practical impact, we further develop a reconstruction algorithm that combines gradient inversion with malicious update strategies. 
Our analysis and experimental results reveal a critical blind spot in FL defenses: both server-side robust aggregation and client-side privacy mechanisms may fail against our proposed attack.
Surprisingly, standard server- and client-side defenses designed to enhance robustness or privacy may unintentionally amplify data leakage. 
Compared to the baseline approach, a mistakenly used defense may instead improve the reconstructed image quality by $10$--$15\%$. 

\end{abstract}

\section{Introduction}\label{sec:introduction}

Federated learning (FL) enables multiple clients to collaboratively train a shared model without sharing raw data, offering a promising alternative to centralized machine learning~\cite{kairouz2021advances, banabilah2022federated}. 
In applied scenarios, regulatory and organizational constraints often prohibit the direct transfer of sensitive data among different entities. 
FL addresses this limitation by keeping clients' data local while leveraging their distributed computational resources.  
A typical FL setup is composed of a central server and a group of participating clients.  
Each client has its own private dataset and performs local training using the global model broadcast by the server. 
Once local updates are completed, clients transmit their models or gradients to the server, where a specified aggregation rule is applied to generate an updated global model. 

Despite its advantages over centralized machine learning, FL still faces key challenges in network security, data privacy, and system robustness~\cite{zhang2024survey, chen2025trustworthy}.
Among the most well-studied threats are server-side data reconstruction attacks and client-side Byzantine attacks. 
In server-side reconstruction attacks, an honest-but-curious server attempts to infer sensitive information about a client's training data~\cite{zhu2019deep, zhao2020idlg, yin2021see}.
Although FL protocols are designed to keep user data local, previous studies have shown that this alone does not guarantee privacy, as significant information can still be leaked through gradient sharing~\cite{zhang2024survey, du2024sok}. 
Specifically, high-fidelity input samples can be reconstructed from client gradients, particularly in the early stages of training when models are poorly generalized~\cite{wang2024more}. 
Since the server has direct access to the gradients uploaded by each client, it is relatively straightforward for this adversary to inspect sensitive information from individual clients. 
Some studies remove the assumption of server honesty, allowing it to modify the structures or parameters of the global model without following standard FL protocols, which can exacerbate data leakage and increase the fidelity of the reconstructed data~\cite{wen2022fishing, boenisch2023curious}.

In practice, assuming a threat model of a fully compromised server as the source of data leakage is oftentimes too strong and unrealistic.  
The detectability of server-side attacks further poses limitations on these threat models~\cite{du2024sok, garovhiding2024hiding}. 
Thus, adversarial attacks on the client side have garnered significant research interest, owing to their practical feasibility~\cite{lyu2022privacy, shi2022challenges}.  
Byzantine attacks are among the most extensively studied and well-known threats in this domain~\cite{shejwalkar2022back, li2023experimental}.
Specifically,  Byzantine attacks involve malicious clients acting arbitrarily or adversarially with the goal of degrading the global model. 
The adversarial client behavior may include data poisoning and gradient manipulation, which can significantly disrupt the convergence of the model when classical aggregation methods such as FedAvg~\cite{mcmahan2017communication} are used. 
Due to the heterogeneous nature of FL, such adversarial behavior is relatively difficult to detect. 
Researchers have found that a small fraction of adversarial clients may significantly degrade model performance~\cite{blanchard2017machine, yin2018byzantine}. 

\begin{figure*}
    \centering
    \captionsetup[subfigure]{aboveskip=20pt}
    \captionsetup{aboveskip=4pt, belowskip=-1pt}
    \subcaptionbox{\label{subfig:a}}[0.32\textwidth]{
        \begin{overpic}[width=\linewidth]{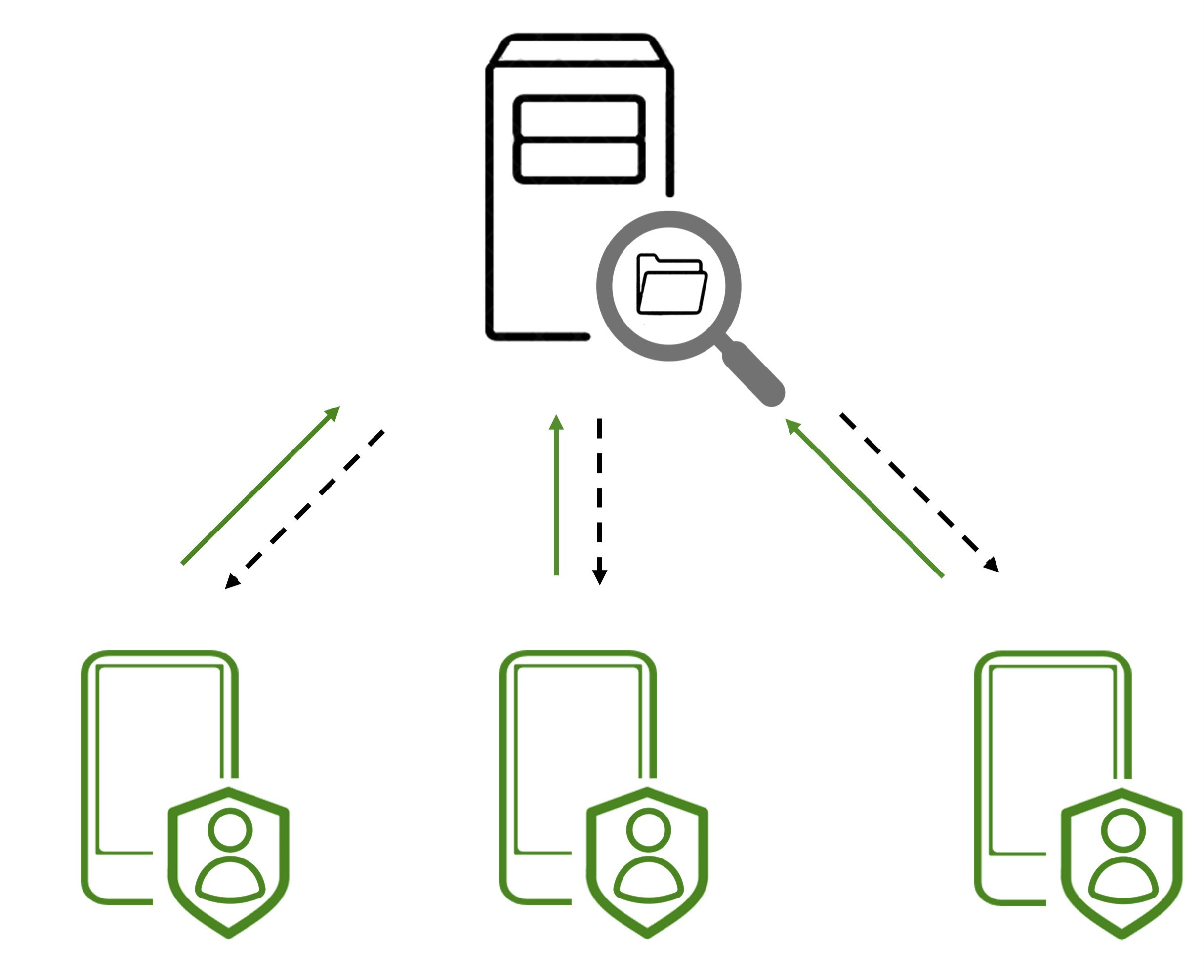}
            \put(5, -3){\intab[0.6]{Client $1$}}
            \put(40, -3){\intab[0.6]{Client $2$}}
            \put(78, -3){\intab[0.6]{Client $m$}}

            \put(60, 74){\intab[0.6]{Honest-but-curious Server}}
            \put(60, 69){\intab[0.6]{Gradient Inversion}}
            \put(59, 62.5){\intab[0.6]{(Ineffective Attack)}}

            \put(20.5, 15){\intab[0.6]{DP}}
            \put(54.9, 15){\intab[0.6]{DP}}
            \put(93.9, 15){\intab[0.6]{DP}}

            \put(2.5, 47){\intab[0.6]{Noisy}}
            \put(2.5, 42){\intab[0.6]{Gradient}}

            \put(76, 47){\intab[0.6]{Global}}
            \put(76, 42){\intab[0.6]{Model}}

        \end{overpic}
    }\hfill
    \subcaptionbox{\label{subfig:b}}[0.32\textwidth]{
        \begin{overpic}[width=\linewidth]{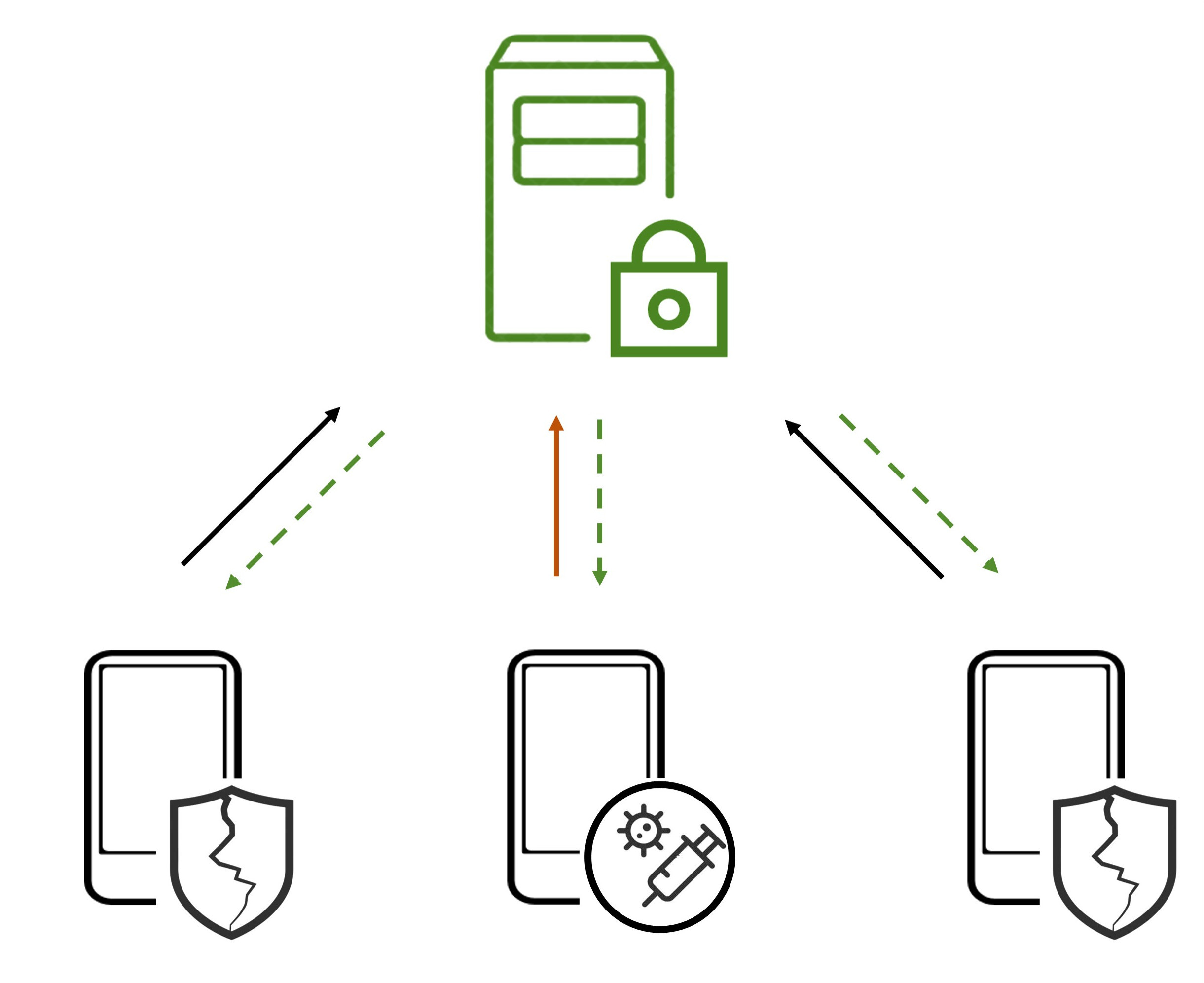}
            \put(5, -5.5){\intab[0.6]{Client $1$}}
            \put(40, -3){\intab[0.6]{Byzantine}}
            \put(40, -8){\intab[0.6]{Attacker}}
            \put(78, -5.5){\intab[0.6]{Client $m$}}

            \put(60, 74){\intab[0.6]{Byzantine-robust}}
            \put(60, 69){\intab[0.6]{Aggregation}}

            \put(24, 29.5){\intab[0.6]{Poisoning~(Ineffective Attack)}}

            \put(70, 52.5){\intab[0.6]{(Robustness $\uparrow$)}}
            \put(76, 47){\intab[0.6]{Global}}
            \put(76, 42){\intab[0.6]{Model}}

            \put(1, 0){\intab[0.58]{Compromised}}
            \put(74, 0){\intab[0.58]{Compromised}}

        \end{overpic}
    }\hfill
    \subcaptionbox{\label{subfig:d}}[0.32\textwidth]{
        \begin{overpic}[width=\linewidth]{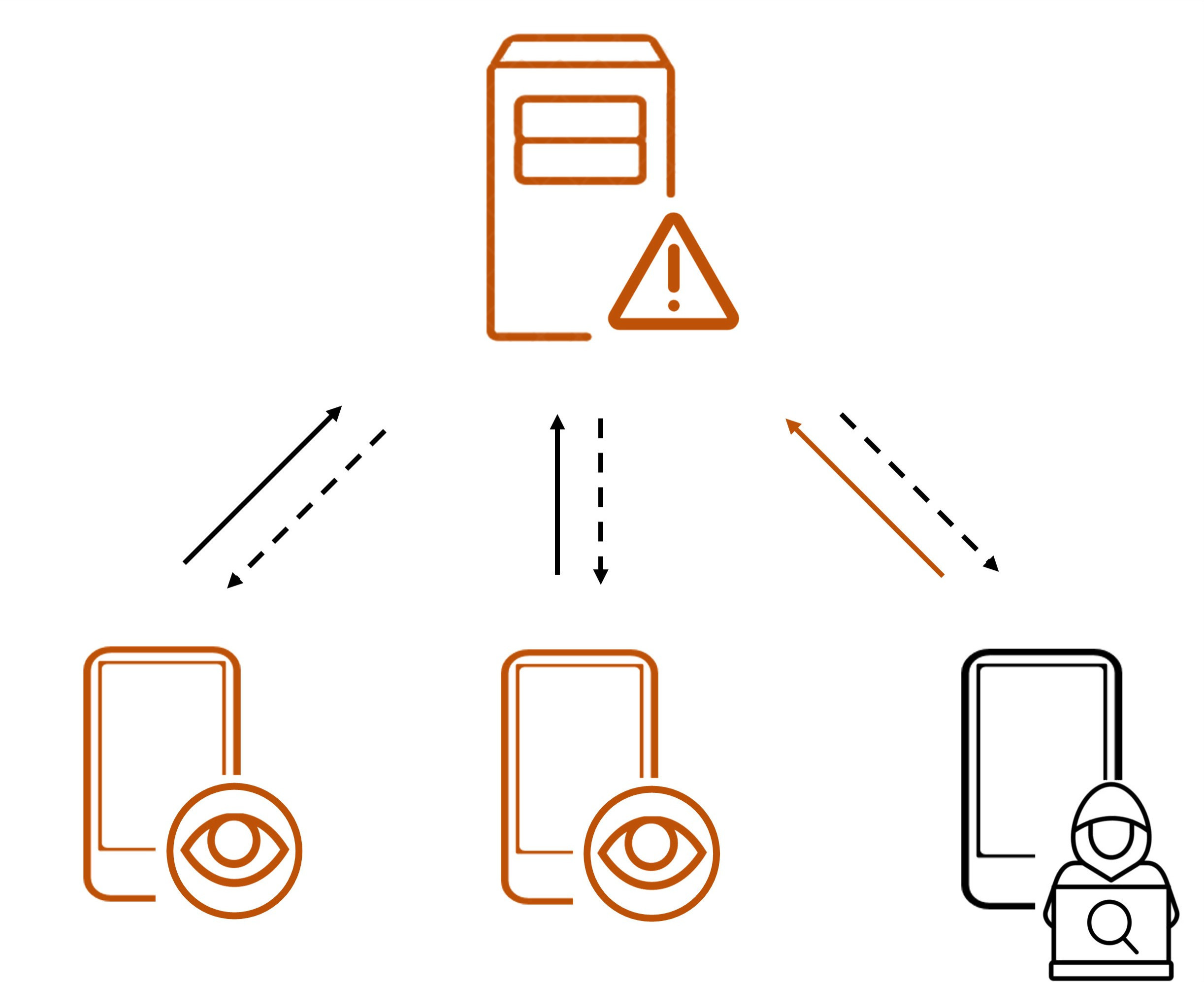}
            \put(5, -3){\intab[0.6]{Client $1$}}
            \put(40, -3){\intab[0.6]{Client $2$}}
            \put(78, -5){\intab[0.6]{Maliciously}}
            \put(76, -10){\intab[0.6]{Curious Client}}

            \put(60, 74){\intab[0.6]{Byzantine-robust}}
            \put(60, 69){\intab[0.6]{Aggregation}}

            \put(59, 61){\intab[0.6]{($\uparrow$ Risk)}}


            \put(20.5, 15){\intab[0.6]{DP}}
            \put(54.9, 15){\intab[0.6]{DP}}

            \put(2.5, 47){\intab[0.6]{Noisy}}
            \put(2.5, 42){\intab[0.6]{Gradient}}

            \put(60, 30){\intab[0.6]{Poisoning}}

            \put(76, 47){\intab[0.6]{Global}}
            \put(76, 42){\intab[0.6]{Model}}
        \end{overpic}
    }
    \caption{
    Illustration of FL systems under different threat and defense scenarios. 
    \textcolor{mygreen}{Green} icons indicate successful defenses. 
    (a)~Clients apply local DP to defend against an honest-but-curious server performing gradient inversion, rendering the attack ineffective.
    (b)~The server employs Byzantine-robust aggregation to neutralize a Byzantine attacker, resulting in ineffective attacks. 
    Broken shield icons represent clients that are presumed to be compromised in the Byzantine threat model. 
    \textcolor{myorange}{Orange} arrow represents the poisoned gradient sent by the attacker. 
    (c)~The proposed threat scenario, where a maliciously curious client sends poisoned updates while attempting to reconstruct private data from peers. 
    Applying standard defenses to enhance robustness and privacy may unintentionally increase the risk of privacy leakage. 
    \textcolor{myorange}{Orange} icons highlight scenarios where privacy risks are increased.
    }
    \label{fig:overview}
\end{figure*}
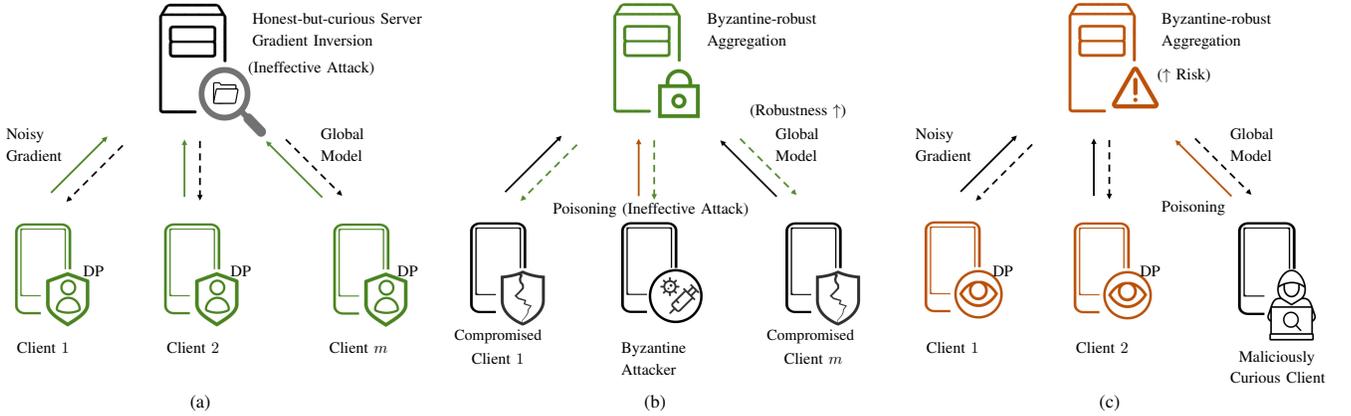

To mitigate these risks, researchers have developed targeted defenses depending on whether the attacker is the server or a client. 
In the case of defending against server-side data reconstruction attacks, client-side defenses have been developed to enhance privacy. 
For example, secure aggregation and homomorphic encryption have been proposed to protect gradients in transit~\cite{munjal2023systematic, rathee2023elsa}. 
Alternatively, differential privacy~(DP) mechanisms inject noise into clients' gradients to reduce the leakage of sensitive information~\cite{wei2021gradient, hu2024does}, offering a trade-off between privacy protection and model accuracy. 
In parallel, server-side defenses such as robust aggregation rules have been proposed to counter Byzantine attacks launched by clients. 
Defenses based on robust statistics~\cite{blanchard2017machine} and anomaly detection~\cite{shejwalkar2021manipulating} aim to filter out corrupted gradients and protect the global model from degradation. 
More advanced defenses include reference-based selection and adaptive filtering~\cite{cao2021fltrust, fang2024byzantine}. 
These robust aggregation rules are often designed and evaluated under specific and idealized Byzantine threat models. 
Theoretical convergence results have also been established for FL with different server-side defenses, although the results may not always translate into practical resilience~\cite{shejwalkar2022back, allouah2023privacy}.

Existing research in FL has predominantly treated gradient-based privacy leakage and Byzantine robustness as separate concerns. 
In this work, we uncover a blind spot at their intersection by introducing a novel client-side threat, namely, the maliciously curious client. 
Such a client-side attacker appears Byzantine from the server's perspective because of its capability of using poisoned datasets or supplying poisoned gradients. 
Meanwhile, this client-side attacker is mainly interested in reconstructing sensitive data from other clients, rather than purely disrupting the training process like Byzantine adversaries. 
This threat model challenges the existing separation between server- and client-side defense techniques. 
Figure~\ref{fig:overview} provides a high-level overview by contrasting standard FL setups with the proposed maliciously curious client scenario.  
We formally establish the threat model by conducting a theoretical analysis, demonstrating that such an attacker can achieve significant reconstruction success under standard FL protocols. 
In particular, server-side Byzantine-robust aggregation defenses are designed to enhance robustness but do not directly protect against data leakage. 
On the other hand, client-side defenses are not optimized to handle malicious gradients and may even end up exposing more vulnerabilities. 
Our extensive experiments reveal that these two types of existing defense may not protect against this threat. 
More critically, we find that combining standard server- and client-side defenses may not effectively mitigate this threat and, in some cases, can even amplify data leakage compared to the non-defense baseline. 
The contributions of the paper are summarized as follows.  

\begin{enumerate}[leftmargin=*]
    \item[\textbullet]
    We introduce a novel client-side threat model for FL, termed the \textit{maliciously curious client}, which manipulates its own gradients with the goal of reconstructing from peer clients.  
    We provide a theoretical analysis demonstrating the feasibility and effectiveness of this attack under standard FL protocols. 
    To the best of our knowledge, this is the first work to analyze the risk of such a threat in FL. 
    \item[\textbullet]
    Going beyond theoretical analysis and threat modeling, we further design a reconstruction algorithm that combines gradient inversion with poisoning strategies to facilitate data reconstruction. 
    \item[\textbullet]
    We evaluate server-side robust aggregation and client-side privacy defenses against this new threat of maliciously curious clients.  
    Our findings reveal a distinct interplay between utility and privacy, where efforts to improve robustness or privacy may unintentionally amplify privacy leakage. 
\end{enumerate}
The remainder of the paper is organized as follows.  
Section~\ref{sec:related_work} reviews related work on privacy and robustness in federated learning.  
Section~\ref{section:preliminaries} introduces the problem setup and notations.  
In Section~\ref{section:algorithm_analysis}, we provide a theoretical analysis of the proposed threat model and detail the corresponding data reconstruction algorithm. 
Section~\ref{sec:exp} describes the implementation and provides experimental results.  
Finally, Section~\ref{sec:conclusion} concludes with a summary and future directions.

\section{Related Work}\label{sec:related_work}

\subsection{Client Privacy Leakage in FL}\label{sec:client_leakage}

Although FL was initially proposed to improve privacy by avoiding raw data exchange~\cite{mcmahan2017communication, li2020federated}, numerous studies have shown that transmitting models or gradients may leak sensitive information~\cite{zhu2019deep, geiping2020inverting, tramer2022truth}. 
These privacy threats can be broadly divided into two categories, namely, inference attacks and data reconstruction attacks. 
Inference attacks aim to extract information about the training data without directly recovering the original inputs.  
For example, membership inference attempts to determine whether a particular example was included in the training dataset~\cite{shokri2017membership, bai2024membership}. 
A participant can passively eavesdrop on the transmitted gradients  and construct a prediction model or actively perform gradient ascent on the non-member example to conduct the attack~\cite{nasr2019comprehensive}. 
On the other hand, a property inference attack can be launched by constructing input datasets that exhibit specific target properties and passing them through the global model. 
The attacker may compare the resulting gradients with those from non-target datasets to infer whether the target property is present in the training data~\cite{melis2019exploiting}. 
These attacks can be conducted by clients or the server. 

Compared to inference attacks, data reconstruction attacks pose a more direct threat of privacy leakage~\cite{zhang2024survey}. 
These attacks can recover training data from clients at the pixel or token level. 
Depending on the type of shared information exploited, these attacks can be broadly classified into parameter-based model inversion attacks and gradient-based data reconstruction attacks. 
The former type indirectly infers training data based on the weights of the global model, typically using generative neural networks to synthesize representative inputs. 
For instance, the generative adversarial network~(GAN)-based method~\cite{hitaj2017deep, wu2024fedinverse} treats the global model as a discriminator to guide the generator in producing inputs resembling those of other clients. 
In contrast, gradient-based attacks directly leverage client-uploaded gradients, enabling more precise recovery of individual training samples through optimization-based inversion.

Our work falls under the second category, which focuses on gradient-based data reconstruction~\cite{zhu2019deep, yin2021see, du2024sok}. 
The attack is often assumed to be launched by the server. 
A well-known optimization-based approach aims to approximate ground-truth inputs through iterative refinement starting from random noise, with objective functions that measure the discrepancy between dummy and real gradients~\cite{zhu2019deep, geiping2020inverting}. 
Another branch of reconstruction attacks analyzes gradients, where the adversary can recursively derive the input from gradients~\cite{fan2020rethinking}. 
In addition, the server may also modify the structure or parameters of the model to increase the success rate of reconstruction~\cite{zhao2023resource, boenisch2023curious}. 
Generally speaking, a compromised server represents a strong and somewhat unrealistic threat, considering that servers are usually deployed in controlled environments, whereas adversarial clients are easier to introduce and harder to detect in decentralized settings~\cite{du2024sok}.
To the best of our knowledge, this is the first work to formally define and theoretically analyze the impact of maliciously curious clients in the context of gradient-based data reconstruction attacks.

%
\subsection{Byzantine Adversaries and Poisoning Attacks}

Robustness is central to trustworthy machine learning, especially in distributed settings where clients can behave inconsistently due to failures, adversarial intent, or unstable communication~\cite{blanchard2017machine, lyu2022privacy}. 
Beyond the privacy leakage in Section~\ref{sec:client_leakage}, another important threat model in FL involves malicious clients that actively interfere with the training process. 
In Byzantine attacks, client attackers may send modified updates to disrupt model convergence~\cite{han2023fedmlsecurity, li2023experimental}. 
For example, Xie et al.~\cite{xie2020fall} designed poisoned updates that cause the aggregated update to have a negative inner product with benign updates, thereby reversing the direction of learning. 
Meanwhile, clients may poison the data or inject specific patterns into the model to cause targeted misbehavior in inference~\cite{li20233dfed, wan2024data}. 
The poisoned model may have a high error rate given the input with backdoor triggers~\cite{shejwalkar2022back}.
Many Byzantine attack scenarios assume that the adversary has full access to training data or gradients of benign clients, essentially presuming a privacy breach from the outset~\cite{xie2020fall, shejwalkar2021manipulating, ozfatura2023byzantines}. 
Consequently, evaluation typically focuses on convergence rate or accuracy degradation, with little attention paid to direct privacy consequences. 

Recently, researchers have started to explore how malicious clients can amplify the membership inference attack with poisoning strategies. 
Tram{\`e}r et al.~\cite{tramer2022truth} showcase that the model can memorize the victim data as ``outliers'' by mislabeling these examples.  
Similarly, Zhang et al.~\cite{zhang2023agrevader} combine the gradient ascent in benign datasets and the gradient descent in mislabeled ones.  
The idea of applying gradient ascent to victim datasets has been widely adopted in membership inference attacks~\cite{nasr2019comprehensive, ma2023loden}. 
We note that these privacy attacks require exact knowledge of the victim dataset to ensure that the adversary can query the model on that dataset. 
This assumption of an available victim dataset aligns with the aforementioned Byzantine attacks, which is fundamentally different from the assumption of unavailable data from benign clients in the data reconstruction setup that this work focuses on.

\subsection{Defenses and Secure Computations}

Various defense mechanisms have been developed to address privacy leakage and robustness threats in FL. 
These techniques span privacy-preserving mechanisms that aim to limit the information exposed in model updates and robustness methods that protect against malicious or unreliable clients.
Differential privacy (DP) is one of the most widely used approaches to limit information leakage~\cite{zhang2022understanding}. 
By injecting calibrated noise into gradients or model parameters, DP limits the contribution of individual data points~\cite{abadi2016deep, wei2021user}. 
Various DP variants exist, including local DP~\cite{wang2020federated}, distributed DP~\cite{kairouz2021distributed}, and central DP~\cite{dubey2020differentially}, each assuming different levels of trust in clients and servers. 
Although DP provides formal privacy guarantees, it often leads to utility loss, especially when clients have small datasets that are not independent and identically distributed~(non-IID)~\cite{zhang2024survey}.
On the other hand, secure aggregation protocols aim to hide individual client updates by aggregating them in an encrypted form~\cite{bonawitz2016practical, kairouz2021distributed}. 
These techniques are effective against an honest-but-curious server, but may introduce computation and communication overhead and struggle with dynamic client participation.
Other cryptographic solutions include secure multiparty computation~\cite{liu2024survey} and homomorphic encryption~\cite{hijazi2023secure}. 
The former technique allows a group of clients to compute global models collaboratively through cryptography, while the latter enables computations on encrypted gradients.
In parallel, trusted execution environments~(TEEs)~\cite{kalapaaking2022blockchain} offer an alternative by executing sensitive computations in a protected enclave.

In terms of robust defenses, numerous aggregation methods have been developed to counter Byzantine clients. 
Classical approaches such as Krum~\cite{blanchard2017machine}, Trimmed Mean~\cite{yin2018byzantine}, and Median~\cite{yin2018byzantine} use robust statistics to suppress outlier gradients. 
More advanced methods reject client updates that do not meet some well-designed criteria~\cite{fang2024byzantine, krauss2023mesas}. 
Meanwhile, clustering-based approaches provide a self-supervised approach to isolate malicious clients~\cite{fereidooni2024freqfed, xu2022byzantine}. 
Notably, these defenses are primarily designed to maintain model robustness rather than to protect clients' data privacy.

Despite their individual strengths, existing defenses may be inadequate when faced with more sophisticated threat models. 
For example, differential privacy and secure aggregation are designed to protect against an untrusted server, but offer little resistance to Byzantine-style poisoning. 
In contrast, robust aggregation methods improve model robustness but do not inherently protect data privacy. 
It is unclear whether the system exposed under a new dedicated threat will be vulnerable with existing well-designed defenses, as in the case of the proposed maliciously curious client with data reconstruction purposes.
As our theoretical analysis and empirical findings will demonstrate, these defenses not only fall short in this setting, but can unintentionally amplify privacy leakage in some cases.

\section{Preliminaries}\label{section:preliminaries}
 
In this section, we review the relevant background to establish the context for our proposed threat model and analysis. 
We cover the standard FL training protocol, along with key privacy-preserving mechanisms and robustness techniques designed to mitigate Byzantine attacks. The symbol conventions used throughout the paper are as follows. We use $[N]$ to denote the integer set $\{1, 2, \dots, N\}$. 
Lowercase boldface letters 
represent column vectors, 
uppercase boldface letters represent matrices, and calligraphic letters denote sets.


\subsection{Federated Learning}

Consider an FL architecture optimized with FedAvg~\cite{mcmahan2017communication}, which is a backbone of commonly adopted FL protocols.  
Denote the $m$th client's dataset by $\mathcal{D}_m =  \{(\x_{m,i},y_{m,i})\}_{i=1}^{N_m}$, where the $i$th example $(\x_{m,i},y_{m,i})$ contains an input-output pair drawn from a distribution $\mathcal{P}_m$.
The local objective function $f_m(\cdot)$ is defined as:
\begin{equation}
    f_m(\w) = \frac{1}{N_m} \sum_{i=1}^{N_m} \ell(\w; \x_{m,i}, y_{m,i}),
\end{equation}
where $\ell$ is a sample-wise loss function quantifying the error of the model with a weight vector $\w \in \mathbb{R}^{d}$ predicting the label $y_{m,i} \in \mathbb{R}$ given an input $\x_{m,i} \in \mathbb{R}^{d_{\text{in}}}$. 
FL aims to optimize the following aggregated problem:
\begin{equation}
    \min_{\w \in \mathbb{R}^{d}} f(\w) = \frac{1}{M} \sum_{m=1}^{M} f_{m}(\w),
\end{equation}
where $M$ is the total number of clients.
In FedAvg, the server will select a subset $\set^{(k)}$ of clients and broadcast a global model $\w[k]$ in each communication round $k$.
Once the model is received by the $m$th client, it will initialize a local model $\w[k,0]_m$ and optimize it with multiple gradient descent steps. 
Consider the set of batch indices $\mathcal{I}^{(k,t)}_m \subseteq [ N_m ]$ with a batch size of $B$, the local model updated at step $t+1$ can be written as 
\begin{equation}
\w[k,t+1]_{m} = \w[k,t]_{m} \!-\! \frac{\eta}{B} \!\sum_{i \in \mathcal{I}^{(k,t)}_m } \! \update[k,t]_{m,i},
\end{equation}
where $\eta$ is the learning rate, and the gradient term $\update[k,t]_{m,i}$ is defined as follows,
\begin{equation}\label{eq:individal_grad}
  \boldsymbol{g}^{(k, t)}_{m,i} \triangleq \nabla \ell (\w[k,t]_m; \x_{m,i}, y_{m,i}). 
\end{equation}
Once each client finishes local update after $\tau$ iterations, the model is uploaded to the server for aggregation. 
Under the FedAvg protocol, the updated model may be written as 
\begin{equation}
    \w[k+1] = \frac{1}{|\mathcal{S}^{(k)}|} \sum_{m \in \mathcal{S}^{(k)}} \w[k, \tau]_{m},  
\end{equation}
where $\mathcal{S}^{(k)}$ represents the selected client set at communication round $k$.


\subsection{Privacy Leakage and Defense}
This subsection reviews the technical details of gradient inversion attack~\cite{zhu2019deep} and the corresponding DP defenses related to privacy leakage. 
Suppose the gradient $\boldsymbol{g}$ of a client is revealed to an honest-but-curious server during information exchange in FL. 
For simplicity, consider the case where the number of local update step $\tau = 1$, 
\begin{equation}
    \boldsymbol{g} \triangleq \frac{1}{B} \sum_{i =1}^{B} \! \nabla \ell (\w; \x_{i}, y_{i}).
\end{equation}
The honest-but-curious server initializes dummy data $\bx_{i} \sim \mathcal{N}(\mathbf{0}, \mathbf{I}_{d_{\text{in}}}), \breve{y}_i \sim \mathcal{N}(0, 1)$, $i \in [B]$, and obtains the dummy gradients $\breve{\boldsymbol{g}}$ using standard forward pass and backward propagation, 
and solves the following optimization problem to match with the shared gradient, 
\begin{equation}\label{eq:optim_problem}
    \{\hat{\x}_{i}, \hat{y}_i\}_{i=1}^{B} = \argmin_{\{\bx_{i}, \breve{y}_i\}_{i=1}^{B}} \| \breve{\boldsymbol{g}} - \boldsymbol{g} \|^{2}. 
\end{equation}
To address privacy vulnerability, various defense strategies have been explored. 
One of the most prominent among them is differential privacy. The details are given as follows.  

\highlight{Differential Privacy~(DP)~\cite{dwork2006differential}. } 
DP provides a formal framework to limit the information that a learning algorithm can disclose about any individual data point in its training set~\cite{abadi2016deep, mohammadi2021differential}. 
Intuitively, an algorithm that satisfies DP guarantees that its output distribution remains nearly unchanged regardless of whether any single individual's data is included in the training set or not. 
The following definition details this notion of indistinguishability.

\begin{Definition}
    A randomized mechanism $M: \mathcal{D} \rightarrow \mathcal{R}$ satisfies $(\varepsilon, \delta)$-differential privacy if, for any two adjacent inputs $D, D^{\prime} \in \mathcal{D}$ and for any subset of outputs $S \subseteq \mathcal{R}$ it holds that
    \begin{equation}
        \prob[M(D) \in S] \leqslant e^{\varepsilon} \prob \left[M\left(D^{\prime}\right) \in S\right]+\delta. 
    \end{equation}
\end{Definition}
A widely adopted approach to train DP models in FL is differentially private stochastic gradient descent (DP-SGD)~\cite{abadi2016deep}. 
In DP-SGD, privacy is achieved by adding calibrated Gaussian noise to the aggregated gradients during training. 
To ensure bounded sensitivity, the algorithm clips the $\ell_2$ norm of each individual example's gradient before averaging. 
Based on the notation in \eqref{eq:individal_grad}, the clipped gradient may be written as 
\begin{equation}\label{eq:clip_grad}
    \tupdate[k, t]_{m,i} \triangleq \frac{\boldsymbol{g}^{(k, t)}_{m,i}}{\max \{1, \| \boldsymbol{g}^{(k, t)}_{m,i} \|_2/C \}} ,
\end{equation}
where $C > 0$ denotes the predefined upper bound on the gradient norm.  
Among the DP variants reviewed in Section~\ref{sec:related_work}, local DP offers the strongest protection, as each client independently enforces DP on its updates before sending them to the server. 
Formally, the update step of the $m$th client with local DP may be represented as 
\begin{equation}
    \w[k, t+1]_{m} = \w[k, t]_{m} - \eta\, \Big(\tfrac{1}{B} \sum_{i \in I^{(t)}} \tupdate[k,t]_{m,i} + \noise \Big), 
\end{equation}
where the noise vector $\noise \sim \mathcal{N}(\mathbf{0},   (C/B)^{2} \sigma^2 \mathbf{I}_{d})$. 
According to the analysis in~\cite{abadi2016deep}, to ensure that the mechanism satisfies $(\varepsilon, \delta)$-DP, the noise strength is suggested to be chosen as
\begin{equation}
    \sigma \geqslant c \frac{q \sqrt{\tau \log (1 / \delta)}}{\varepsilon},
\end{equation}
where $q=B / N_m$ is the sampling rate, $\tau$ is the total number of local update steps, and $c$ is a nonnegative constant. 

\subsection{Robust Aggregation Against Byzantine Clients}\label{section:byzantine_defense}
Many robust aggregation method  identifies a benign client set $\mathcal{S}^{(k)}_{\Lambda}$ and performs FedAvg on the selected gradients, i.e., 
\begin{equation}
    \w[k+1] = \frac{1}{| \mathcal{S}^{(k)}_{\Lambda} |} \sum_{m \in \mathcal{S}^{(k)}_{\Lambda}} \w[k, \tau]_{m}. 
\end{equation}
We provide examples of state-of-the-art Byzantine-robust defenses as follows. 
\begin{enumerate}[label=(\roman*)]
    \item The \textit{frequency analysis-based federated learning (FreqFed)}~\cite{fereidooni2024freqfed} applies the discrete cosine transform to the received gradients and obtains transformed low-frequency coefficients, which are used to identify a benign client set $\mathcal{S}_{\Lambda_{\text{F}}}$ via the hierarchical density-based spatial clustering~(HDBSCAN) algorithm. 
    \item \textit{Byzantine-robust averaging through local similarity in decentralization (Balance)}~\cite{fang2024byzantine} uses a norm-based criterion to identify the benign set $\mathcal{S}_{\Lambda_{\text{B}}}$, excluding clients whose update norms deviate significantly from a reference calculated on the server. 
    \item \textit{Divide and Conquer (DnC)}~\cite{shejwalkar2021manipulating} defense calculates outlier scores based on the eigenanalysis of gradient matrices and constructs a desirable set $\mathcal{S}_{\Lambda_{\text{D}}}$. 
    \item \textit{SignGuard}~\cite{xu2022byzantine} filters out the outlier gradients and then selects client updates whose gradient signs align with the majority of others.
Clients that pass both filtering stages are included in  $\mathcal{S}_{\Lambda_{\text{S}}}$. 
    \item The well-adopted \textit{Krum} and \textit{Multi-Krum}~\cite{blanchard2017machine} choose client updates that are closest to their neighbors to form a selected set $\mathcal{S}_{\lambda_{\text{K}}}$. 
\end{enumerate}
More details of these defenses can be found in Appendix~\ref{app:robust_agg}.
In the next section, we will investigate the potential impact of a client-side data reconstruction attack.   
\section{Impact of Maliciously Curious Clients}
\label{section:algorithm_analysis}

In this section, we first define the proposed threat model. We then analyze its risk by presenting a theoretical upper bound on the reconstruction error, from which we discuss the impact of various factors and conjecture the potential attacks that a maliciously curious client can adopt.

\subsection{Threat Model}

We consider a single \textit{maliciously curious client} as the threat model. The objectives and assumptions for both the client attacker and the system defenders are outlined below. 

\highlight{Attacker's Goal. } 
The client attacker aims to reconstruct private training data belonging to other benign clients. 

\highlight{Attacker's Capability. } 
The client attacker participates in training as a legitimate client and does not collude with other participants.
It can manipulate its own gradient updates before uploading to the server and eavesdrop on changes in the global model to facilitate data reconstruction. 
The attacker adheres to the model architecture specified in the training protocol and does not modify it.  
To persist in training and continually receive broadcast models, the attacker may also employ a Sybil strategy, rejoining the training process under different identities if excluded. 

\highlight{Attacker's Knowledge. } 
The client attacker knows the architecture and parameters of the global model, the number of clients in the system, and the batch size used in the FL optimization.  
It has full access to its own local dataset but no access to the data or updates of other clients. 

\highlight{Defender's Goal. } 
In the context of this paper, In the context of this paper, the FL system defenders include a trustworthy central server and benign clients, excluding the single maliciously curious client. 
The server defender's goal is to preserve the integrity of the model, for example, to prevent training failure and backdoor injection. 
Benign clients aim to preserve the confidentiality of their private training data. 

\highlight{Defender's Capability. } 
All defenders operate under a standard FL protocol.
The server can apply Byzantine-robust aggregation rules and select client subsets for aggregation.
Benign clients may implement local differential privacy mechanisms before transmitting their updates. 
In the following, we highlight the differences and similarities between the proposed maliciously curious client and other existing threat models in the literature. 
Compared to Byzantine clients in Byzantine robust FL studies~\cite{shejwalkar2022back, li2023experimental, fang2024byzantine}, the proposed maliciously curious client has the same capability to modify its uploaded gradients or model. However, it is strictly weaker in terms of knowledge, as it does not assume access to the data or gradients of benign clients. Moreover, we consider only a single adversarial client, whereas Byzantine attacks often involve multiple colluding entities~\cite{xie2020fall, shejwalkar2021manipulating, li20233dfed}. This makes our threat model more practical in scenarios where only one adversary is present. 
Additionally, compared to the honest-but-curious participant in the literature~\cite{hitaj2017deep, wu2024fedinverse}, the proposed client attacker has the same goal of reconstructing benign clients' data. 
It can be more harmful in terms of poisoning capacity, with the risk of being detected and filtered out by a trustworthy server.

\subsection{Quantifying Attack Impact}\label{sec:quantify_impact}

We start with a gradient-based reconstruction function $\recon_{i}: \mathbb{R}^{d} \rightarrow \mathbb{R}^{d_{\text{in}}}$,
which takes a gradient vector or model update as input and estimates a data point $\hx_i$. 
The restoration of labels $\hat{y}_i$ can be analyzed in the same way and thus excluded in this paper. 
Within the context of client-side observation, the gradient may be represented as the difference between successive broadcast model weights, 
\begin{equation}
    \update[k] \triangleq \frac{1}{\eta}  \left[ \w[k] - \w[k+1] \right] .
\end{equation}
Suppose that the adversary aims to estimate $N$ data points $\X = [\x_1, \dots, \x_{N}]^{\top}$. 
The root mean square error~(RMSE) in expectation is evaluated as 
\begin{equation}
    \text{RMSE}(\hX^{(k)}) = \sqrt{ \frac{1}{d_{\text{in}} N} \mathbb{E} \| \Recon(\update) -  \X\|^2_{\text{F}} } .  
\end{equation}
where $\hX^{(k)} = [\hx^{(k)}_{1}, \dots, \hx^{(k)}_{N}]^\top$ represent the reconstructed training examples in the $k$th communication round, $\Recon(\cdot) \triangleq [\recon_{1}(\cdot), \dots, \recon_{N}(\cdot)]^\top$, and $\| \cdot \|_{\text{F}}$ is the Frobenius norm.  
We now introduce two standard assumptions widely used in federated optimization analysis~\cite{stich2019local, huang2021fl}.

\begin{Assumption}\label{assumption:L_continuous}
    The gradients of the objective functions $f_{m}$ are Lipschitz-continuous.  
    $\forall \; \w_1, \w_2 \in \mathbb{R}^d$, $m \in [M]$,  
    there exists some nonnegative $L_{g}$ such that:
    \begin{equation}
        \|\nabla f_{m}(\w_1) - \nabla f_{m}(\w_2) \|_2 \leqslant L_{g} \, \|\w_1 - \w_2 \|_2.
    \end{equation}
    We also assume that the reconstruction functions $\recon{i}$ are Lipschitz-continuous.  
    $\forall \; \boldsymbol{g}_1, \boldsymbol{g}_2 \in \mathbb{R}^d$, $i \in [N]$,  
    there exists some nonnegative $L_{\psi}$ such that:
    \begin{equation}
        \| \recon_{i}(\boldsymbol{g}_1) - \recon_{i}(\boldsymbol{g}_2) \|_2 \leqslant L_{\psi} \, \| \boldsymbol{g}_1 - \boldsymbol{g}_2 \|_2.
    \end{equation}
\end{Assumption}

\begin{Assumption}\label{assumption:bounded_data}
    Input data and reconstructed data have the same norm $\upsilon$, i.e., 
    $\| \x_{i} \|_2 = \| \recon_{i}( \update )  \|_2 = \upsilon$, $\forall\, i \in [N]$. 
\end{Assumption}

We now establish the upper bound for the reconstruction error under FedSGD, which is a special case of FedAvg with the number of local iterations $\tau = 1$. 

\begin{Theorem}\label{theorem:non_convex}
    Consider FedSGD with local DP against the client-side reconstruction attack. 
    Suppose that the conditions in Assumptions~\ref{assumption:L_continuous}--\ref{assumption:bounded_data} hold.  
    Let the learning rate $\eta = O\left(\frac{1}{L_{g}}\right)$, 
    then at the $k$th communication round, the expected reconstruction error satisfies the following inequality, 
\begin{equation}
    \textrm{\normalfont RMSE}(\hX^{(k)}) \! \leqslant \!\min\!\left\{  O\!\left(\rho_0^{(k)}\!\Delta^{(k+1)}\! + \! \rho_1^{(k)}\!\sigma \!+\! e^{(0)}\right) \!,\!  \frac{2\upsilon}{\sqrt{d_{\text{in}}}} \right\},
\end{equation}
where $\Delta^{(k+1)} = [ f(\w[0]) - f(\w[k+1])]^{1/2}$ is the objective gap, 
$\rho_0^{(k)} = 2 L_{\psi} \sqrt{2 L_g k}$ and 
$\rho_1^{(k)} = \frac{2 \sqrt{2d} L_{\psi}  C k }{\sqrt{M} B}$ are the time-variant terms,
and $e^{(0)} = \max_{i} \| \x_i - \recon_{i}(\update[0])  \| $ is the base error term. 
\end{Theorem}
\begin{proof}
    See Appendix~\ref{app:missing_proof}. 
\end{proof}\\
We note a few observations and present the corresponding insights as follows. 

\begin{Remark}\label{remark:training_progress}
    (Impact of training progress.) 
    The reconstruction of training data will be more accurate for FL models that are less trained.
    This is reflected mainly in the first term $\rho_0^{(k)} \Delta^{(k+1)}$ in a multiplicative way. First, the larger the objective gap $\Delta^{(k+1)}$ is (i.e., the more the FL model has converged), the larger the data reconstruction error will be.  
    This theoretical result is consistent with empirical observations in server-side reconstruction attacks that a better converged model may lead to reduced privacy leakage due to gradient inversion~\cite{wang2024more, du2024sok}.
    Second, the reconstruction error increases with the time/communication round
    index $k$, reflected by time-variant coefficients $\rho^{(k)}_0$.
    
\end{Remark}

\begin{Remark}\label{remark:impact_noise}
    (Impact of additive noise level.) 
    The impact of the noise level $\sigma$ is twofold. 
    On the one hand, a higher noise level increases the reconstruction error, which is expected by the second term $\rho^{(k)}_{1} \sigma$. 
    This is consistent with the well-established DP result, which states that a higher noise level leads to a tighter privacy budget~\cite{abadi2016deep, mohammadi2021differential}.   
    On the other hand, a higher noise level will negatively affect the progress of the training, leading to a decrease in the objective gap~$\Delta^{(k+1)}$. 
    The negative impact of DP on FL training has also been empirically observed~\cite{boenisch2023reconstructing}.  
    Depending on the choice of the noise level, DP may not always be preferable in terms of reducing the reconstruction error compared to the non-defense baseline. 
\end{Remark}

\begin{Remark}\label{remark:initial_cond}
    (Impact of initial conditions.) 
    The base error term $e^{(0)}$ reflects the reconstruction error at the early stage of training. 
    It is influenced by several factors, including the initialization of the model, the characteristics of the dataset, and the choice of the reconstruction algorithm. 
    Previous experimental studies on server-side reconstruction attack  have also found that these key factors can significantly affect the quality of reconstruction~\cite{wang2024more, du2024sok}. 
\end{Remark}

\begin{Remark}\label{remark:num_clients}
    (Impact of the number of clients.) 
    When there is no noise injection, the training progress generally improves with more FL participants both in theory~\cite{stich2019local} and in practice~\cite{mcmahan2017communication}.
    This improvement may be associated with a larger objective gap, $\Delta^{(k+1)}$, which leads to a higher reconstruction error.
    Furthermore, the base error term $e^{(0)}$ reflects the maximum reconstruction error across all training examples.
    As the number of clients increases, the adversary encounters a larger and more diverse dataset, potentially increasing the worst-case reconstruction error. 
    With the two effects added, we expect the reconstruction error to increase with the number of clients. 
\end{Remark}

\begin{Remark}
    (Extension to FedAvg.) 
    DP clipping effect may depend on the distribution of the gradients and will further introduce clipping bias~\cite{chen2020understanding}.
    Although Theorem~\ref{theorem:non_convex} is based on FedSGD, it can be extended to FedAvg with local DP by incorporating the interaction between client heterogeneity and clipping bias~\cite{zhang2022understanding}. 
    The empirical studies of FedAvg are further discussed in Section~\ref{sec:exp}.
    A comprehensive analysis is left for future study. 
\end{Remark}
In the next subsection, we discuss how such an attacker can leverage these findings to its advantage. 

\subsection{Roles of Maliciously Curious Client }
Based on the theoretical observations mentioned in Section~\ref{sec:quantify_impact}, we conjecture how a maliciously curious client can design its attack to facilitate reconstruction.  
\begin{enumerate}[itemindent=1.5em]
    \item[(C-1)] The client attacker can use poisoning to negatively affect the progress of training, effectively reducing the objective gap $\Delta^{(k+1)}$. 
    According to Remark~\ref{remark:training_progress}, fewer participants lead to better reconstruction. 
    \item[(C-2)] Consider benign clients adopting DP-based defense.   
    According to Remark~\ref{remark:impact_noise}, the potential effects could occur on multiple fronts. 
    Higher DP noise may increase training error and impede convergence, but excessive noise also provides strong privacy protection. 
    Moreover, the interaction between DP noise and poisoning can further impede convergence, making reconstruction easier. 
    Therefore, applying DP does not always guarantee improved privacy under this threat model. 
    \item[(C-3)] 
    Consider that the server chooses a Byzantine-robust aggregation method described in Section~\ref{section:byzantine_defense}. 
    If the defense is ineffective, the scenario falls back to C-1.  
    If the server defends well against poisoning attacks, it will lead to a decreased number of participants. 
    According to Remark~\ref{remark:num_clients}, the reconstruction error may be smaller compared to the baseline without using defense. 
    Moreover, Byzantine-robust defenses can slow down the training~\cite{blanchard2017machine, allouah2023privacy}, which may also result in better reconstruction based on Remark~\ref{remark:training_progress}. 
\end{enumerate}
We note that cryptographic defenses such as homomorphic encryption and secure multiparty aggregation do not directly protect against a client-side attacker. Since the maliciously curious client observes the broadcast global model, these cryptographic schemes do not alter the threat landscape and are thus not considered as distinct defenses in the discussion. 

In the following, we present the realization of the client-side data reconstruction for input data $\x_i$'s. 
The attack procedure combines gradient inversion and model poisoning. 
In particular, the client attacker first participates in two consecutive rounds of training and obtains the global models $\w^{(k)}$ and $\w^{(k+1)}$ broadcast by the server to calculate the aggregated gradient $\update[k] = (\w[k] - \w[k+1]) / \eta$.
To begin reconstruction, the adversary uses a pretrained autoencoder~\cite{baldi2012autoencoders}, denoted by a pair of encoders and decoders $(\enc, \dec)$, and randomly initializes the code $\{ \z^{(k)}_{m,i} \,|\, m \in [M], i \in [B] \}$.
The autoencoder generates a dummy data matrix $\bx^{(k)}_{m,i} = \enc(\z^{(k)}_{m,i})$ and calculates the dummy gradient $\breve{\boldsymbol{g}}^{(k)}_{m,i}$ with forward pass and backpropogation. 
Meanwhile, the client attacker may simulate the server-side aggregation process by iterating on a set of hypothesized surrogate functions $\{\Phi_{q}(\cdot)\}_{q=0}^Q$.
A simple example is $\Phi_{0}(\{ g_m \}_{i=j}^M) = \frac{1}{M} \sum_{i=1}^M g_i$, which simulates FedAvg aggregation.  
For consistency, the function $\Phi_{q}$ operates on the scaler inputs and can  also accept vector series by performing coordinate-wise operations. 
We provide a few more examples as follows. 
\begin{Example}
   \normalfont(Soft Median~\cite{geisler2021robustness}) 
    To simulate the effect of median aggregation~\cite{yin2018byzantine}, the adversary may adopt a soft median function.  
    Specifically, given an input series $\{g_{j}\}_{j=1}^{M}$, 
    The soft median may be designed as a weighted average, where each weight depends on how far a value is from an estimate of the median $\tilde{g}$,
    \begin{subequations}
    \begin{align}
        & \quad b_m  = \frac{\exp \left(-| g_m - \tilde{g} | / T\right)}{\sum_{j} \exp \left(-| g_j - \tilde{g} | / T\right)}, \\
        & \Phi_{1}\left(\left\{g_m\right\}_{m=1}^{M}\right)  = \sum_{m=1}^{M} b_{m} g_{m}. 
    \end{align}
    \end{subequations}
    Here, $T>0$ is the temperature parameter that controls the sensitivity of the weights to deviations from the median.  
    The initial estimate $\tilde{g}$ may be obtained by calculating the mean value of the series. 
    The final aggregation function is performed coordinate-wise on the gradients. 
\end{Example}

\begin{Example}
    \normalfont(Pseudo-Krum) 
Suppose the total number of attackers is $A$, a Krum defender chooses one client update that is closest to its neighbors~\cite{blanchard2017machine}. 
Formally, define the Krum score for each client $m$ as
\begin{equation}
s_{m} = \sum_{n \in \mathcal{N}_{m}} \left\| \boldsymbol{g}_{n} - \boldsymbol{g}_{m} \right\|^2,
\end{equation}
where $\mathcal{N}_{m}$ is the set of $M - A - 2$ clients closest to client $m$ in terms of Euclidean distance.
Then, the Krum aggregation rule selects the update with the lowest score:
\begin{equation}\label{eq:krum_}
\widehat{\Phi}_{2}(\{\boldsymbol{g}_{m}\}_{m=1}^{M}; A) = \boldsymbol{g}_{i^\star}, \quad i^\star = \argmin_{i} s_i.
\end{equation}
On the attack side, it does not solve \eqref{eq:krum_} directly. 
Instead, it randomly selects one vector from the input series, 
\begin{equation}
    \Phi_{2}(\{\boldsymbol{g}_{m}\}_{m=1}^{M}) = \boldsymbol{g}_i, \quad i \sim \mathcal{U}(1, \dots, M), 
\end{equation}
where $\mathcal{U}(\cdot)$ denotes the discrete uniform distribution.  
The surrogate function $\Phi_{2}(\cdot)$ is called a pseudo-Krum function. 
\end{Example}

We note that the maliciously curious client does not know the exact defense strategy employed by the server. 
Therefore, the attacker optimizes over the set of surrogate functions $\{\Phi_{q}(\cdot)\}_{q=0}^Q$. 
Accordingly, the objective function of the maliciously curious client may be written as
\begin{equation}
    \min_{ \substack{ q \in \{0, 1, \dots, Q\} \\ \{\z^{(k)}_{m,i} \,|\, m \in [M],\, i \in [B]\} }} \| \Phi_{q}(\{\breve{\boldsymbol{g}}^{(k)}_{m}\}_{m=1}^M) - \update[k]  \|^{2}. 
\end{equation}
After conducting the reconstruction attack, the client attacker will proceed with its own optimization. 
The client attacker modifies its message before submission using a poisoning function $p(\cdot)$. 
Below, we also provide a few examples of the poisoning attack without requiring data from the adversary's peers.

\begin{Example}\label{example:sf} 
    \normalfont(Sign Flipping~\cite{li2019rsa}) 
    Given an input gradient~$\update$, the client attacker flips the signs coordinate-wise.
    In addition, the poisoning strength can be controlled by a scaling factor $\kappa > 0$.
    Precisely, the poisoning function is designed as follows,
    \begin{equation}
        p_{1}(\update) = - \kappa \update. 
    \end{equation}
    In the FL optimization, the client attacker provides the negative gradient that may mislead the server to increase the loss. 
    This behavior is expected in Conjecture~C-1. 
\end{Example}

\begin{Example} 
    \normalfont(Gaussian Attack~\cite{blanchard2017machine}) 
    The client attacker will choose to submit a Gaussian random vector in this case. 
    The poisoning function may be written as
    \begin{equation}
        p_{2}(\update) = \boldsymbol{\zeta}, \quad \boldsymbol{\zeta} \sim \mathcal{N}(\mathbf{0},   \sigma^2_{\zeta}\mathbf{I}),
    \end{equation}
    where $\sigma_{\zeta}$ is the noise level. 
    Within the context of FL optimization, the client attacker provides the random gradient direction to the server to increase the loss. 
\end{Example}

\begin{Example} \label{example:backdoor} 
    \normalfont(Backdoor Attack~\cite{li20233dfed}) 
    The poisoning may also happen at the data level. 
    Given a reference dataset $\mathcal{D}_{a} = \{ (\x_{i}, y_{i}) \}_{i=1}^{B}$, 
    the poisoning processes each data pair as follows, 
    \begin{equation}
        p_{3}(\x_i, y_{i}) = ( \x_i + \mathbf{p}_i, y'_i), \; \forall\, i \in [B], 
    \end{equation}
    where $\mathbf{p}_i \in \mathbb{R}^{d_{\text{in}}}$ is an artificial pixel pattern, $y'_i$ is a label assigned by the adversary. 
    The backdoor attack can be designed in different ways with well-designed patterns and can be combined with a gradient-based attack~\cite{zhang2022neurotoxin}.    
\end{Example}

\section{Evaluation of Maliciously Curious Attacks}\label{sec:exp}

To validate the theoretical analysis in Section~\ref{section:algorithm_analysis}, we experimentally evaluate the risk of the proposed maliciously curious threat model and the associated data reconstruction attacks through various factor analyses. 
The task being evaluated is the federated training of convolutional neural networks for image classification, using the Dirichlet non-IID Fashion-MNIST~(F-MNIST) and MNIST datasets~\cite{hsu2019measuring}. 
We consider $M=10$ clients, optimizing the models with $\tau=5$ local iterations and a batch size of $B=100$. 
The detailed experimental setups can be found in Appendix~\ref{app:setup}.


\subsection{Impact of Training Progress}\label{sec:impact_progress}

\afterpage{%
\begin{figure}[!tb]
    \begin{overpic}[width=\linewidth, height=0.5\linewidth]{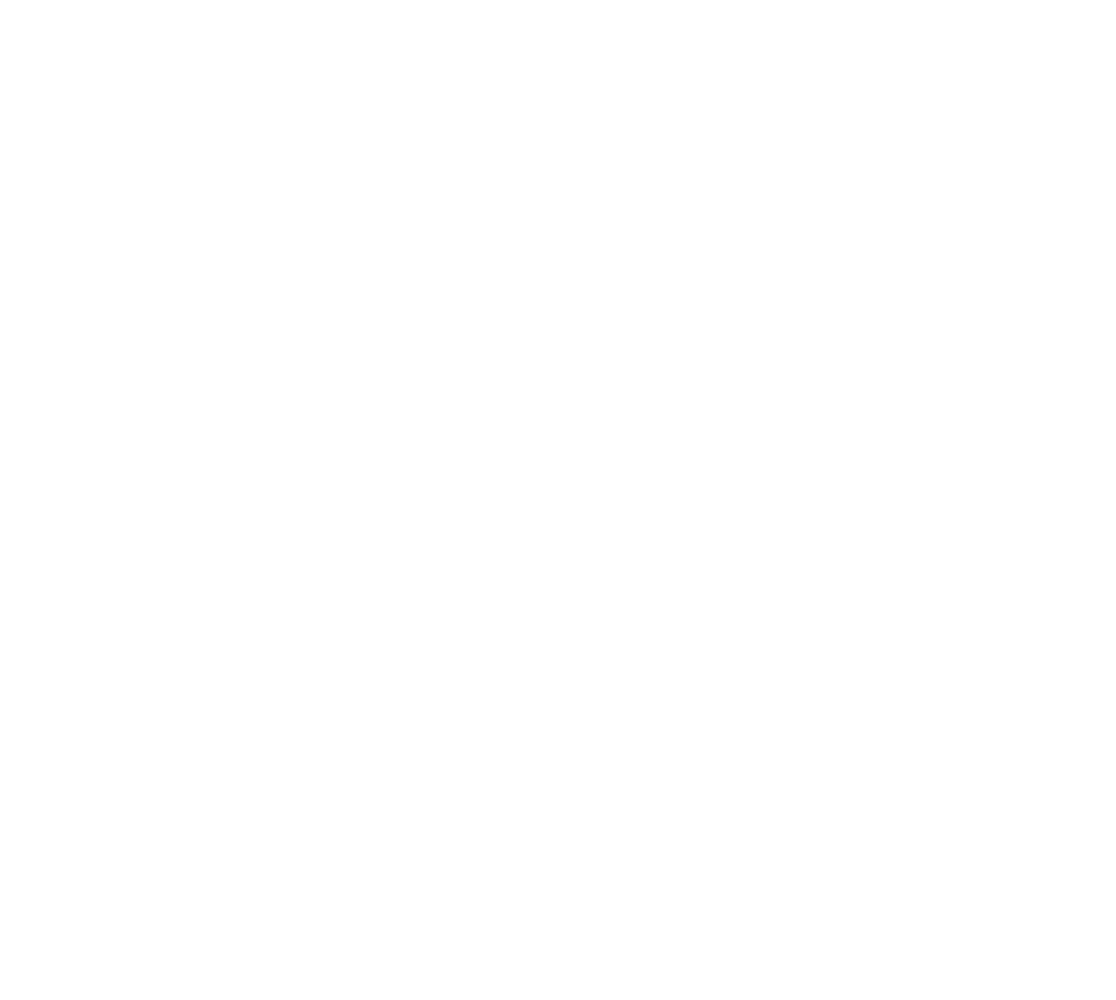}

    \put(23, 48){\intab{\bf MNIST}}
    \put(70, 48){\intab{\bf F-MNIST}}
    
    \put(49.4, 5){\includegraphics[width=0.51\linewidth]{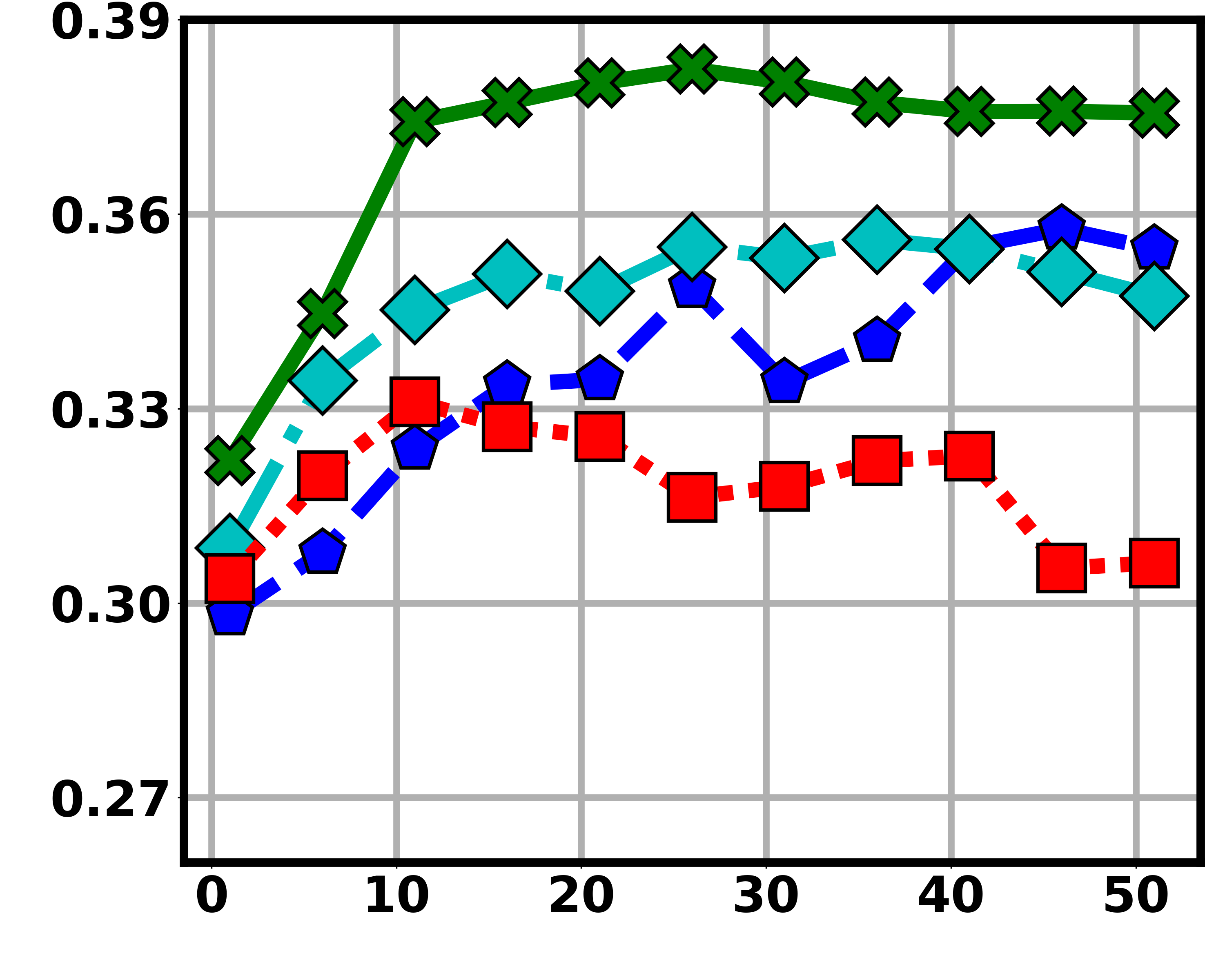}}
    \put(0.5, 5){\includegraphics[width=0.51\linewidth]{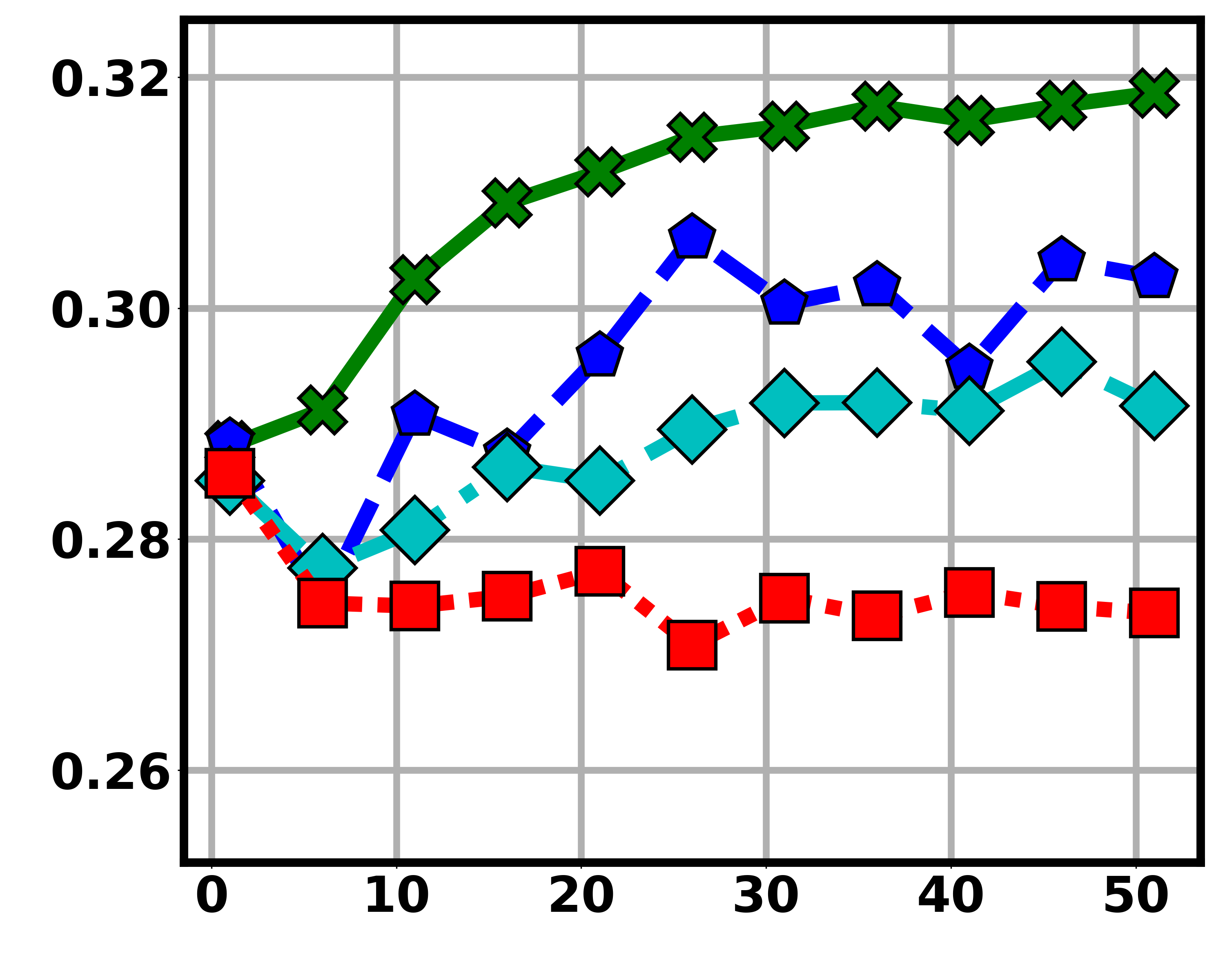}}

    \put(27, -5){\includegraphics[width=0.56\linewidth]{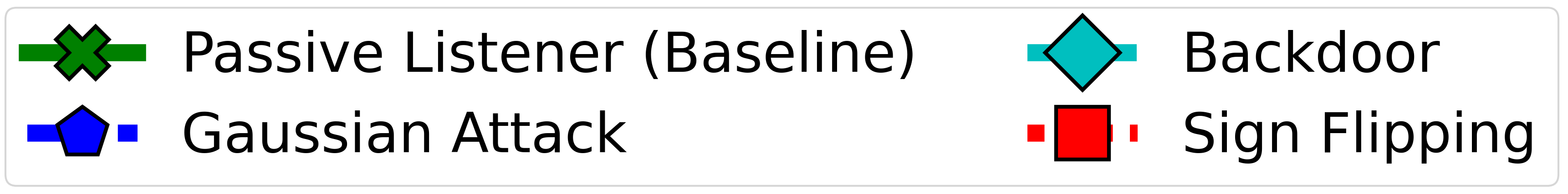}}

    \put(16, 5){\intab{\# of comm. rounds}}
    \put(65, 5){\intab{\# of comm. rounds}}

    \put(0, 21){\rotatebox{90}{\intab{ RMSE}}}

    \end{overpic}
    \vspace*{1pt}

    \caption{
        Reconstruction error versus the number of communication rounds under different poisoning strategies. 
        Poisoning-based attacks reduce the RMSE achieved by the Passive Listener baseline through slowing down the convergence of the FL model training, with Sign Flipping being the most effective. 
        This observation is consistent with Conjecture C-1, confirming that poisoning can amplify reconstruction capability.
    \label{fig:claim1}}
\end{figure}
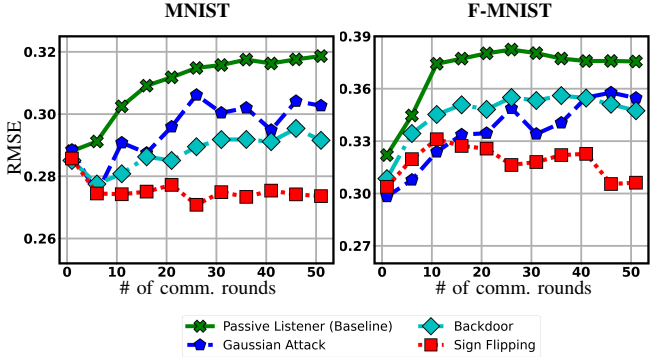

\begin{table}[!tbp]
    \centering 
    \caption{Test accuracy ($\%$) under different poisoning functions. }\label{tab:acc_c1}
    \begin{tabular}{p{0.1\linewidth}p{0.16\linewidth}p{0.16\linewidth}p{0.16\linewidth}p{0.16\linewidth}}
        \toprule 
        Task      & \intab[0.9]{Passive Listener} & \intab[0.9]{Backdoor} & \intab[0.9]{Gaussian} & \intab[0.9]{Sign Flipping} \\ \midrule 
         \intab[0.9]{MNIST} & $  81.4 $ & $57.3$ & $68.5$ & $48.1 $ \\  \cmidrule{1-5} 
         \intab[0.9]{F-MNIST} & $ 72.5 $ & $59.4$ & $56.6$ & $46.7 $   \\ \bottomrule
    \end{tabular}
\end{table}
}

To validate Conjecture~C-1 and support Remark~\ref{remark:training_progress}, we evaluate the effect of different poisoning attacks on reconstruction performance without defense by using FedAvg. 
Specifically, we consider Sign Flipping, Backdoor attack, and Gaussian attack stated in Examples~\ref{example:sf}--\ref{example:backdoor}. 
We contrast with a \textit{Passive Listener} baseline, where the poisoning function $p(\cdot)$ reduces to the identity mapping.
Figure~\ref{fig:claim1} shows the reconstruction error of the reconstructed images as a function of the total number of communication rounds.
Across both tasks, the Passive Listener baseline exhibits a generally increasing trend in reconstruction error, consistent with Remark~\ref{remark:training_progress} that better model convergence may lead to greater reconstruction difficulty.
In contrast, all three poisoning-based attackers effectively hinder training progress and lead to lower reconstruction errors compared to the baseline. 
Notably, the Sign Flipping attack, which is known to be more disruptive to model convergence compared to the other two poisoning attacks, tends to achieve the lowest reconstruction error. 
This aligns with our Conjecture~C-1 that maliciously curious clients can exploit poisoning to slow model convergence and amplify privacy leakage.

TABLE~\ref{tab:acc_c1} reports the test accuracy at the end of the training. 
The results show that different poisoning strategies degrade model performance to varying degrees, with Sign Flipping yielding the lowest test accuracy.
It is worth noting that test accuracy does not directly correspond to reconstruction error.  
For example, while F-MNIST has lower test accuracy overall, its reconstruction error is typically higher than MNIST due to increased dataset complexity.


\subsection{Impact of Client-Side DP Noise Level }

\afterpage{%
\begin{figure}[!tb]
    \begin{overpic}[width=\linewidth, height=0.95\linewidth]{fig/4x4.pdf}
    
    \put(30.5, 0){\includegraphics[width=0.42\linewidth]{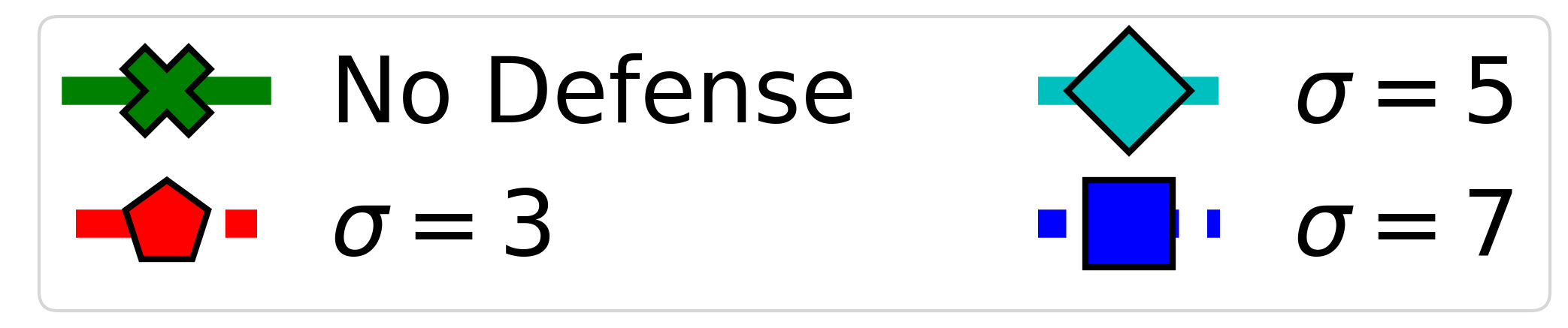}}
    
    \put(0, 62.5){\rotatebox{90}{\intab{\bf Passive Listener}}}
    \put(0, 24){\rotatebox{90}{\intab{\bf Backdoor}}}

    \put(24.5, 93){\intab{\bf MNIST }}
    \put(72, 93){\intab{\bf F-MNIST }}

    \put(52, 52.){\includegraphics[width=0.49\linewidth]{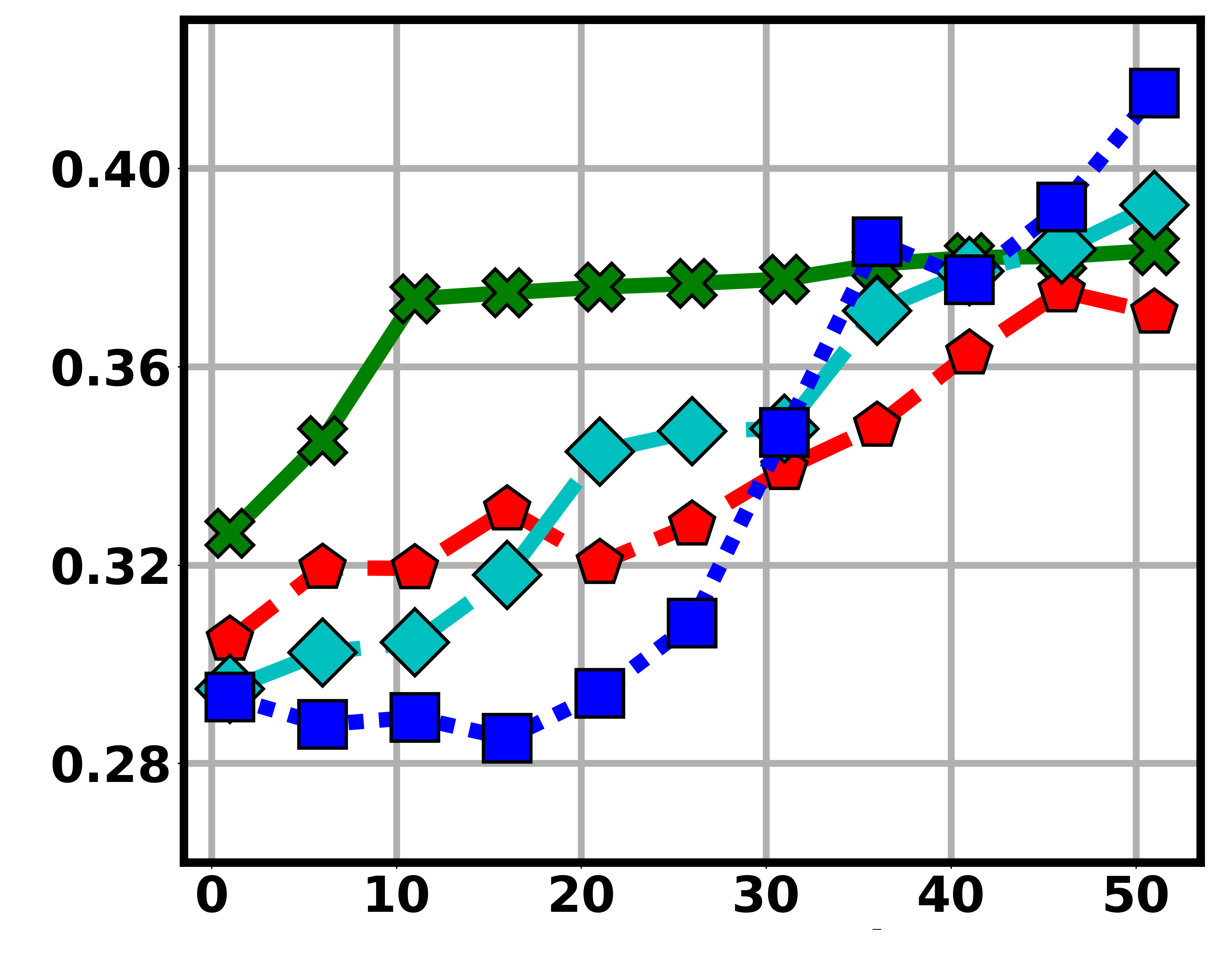}}
    \put(4.8, 52.){\includegraphics[width=0.49\linewidth]{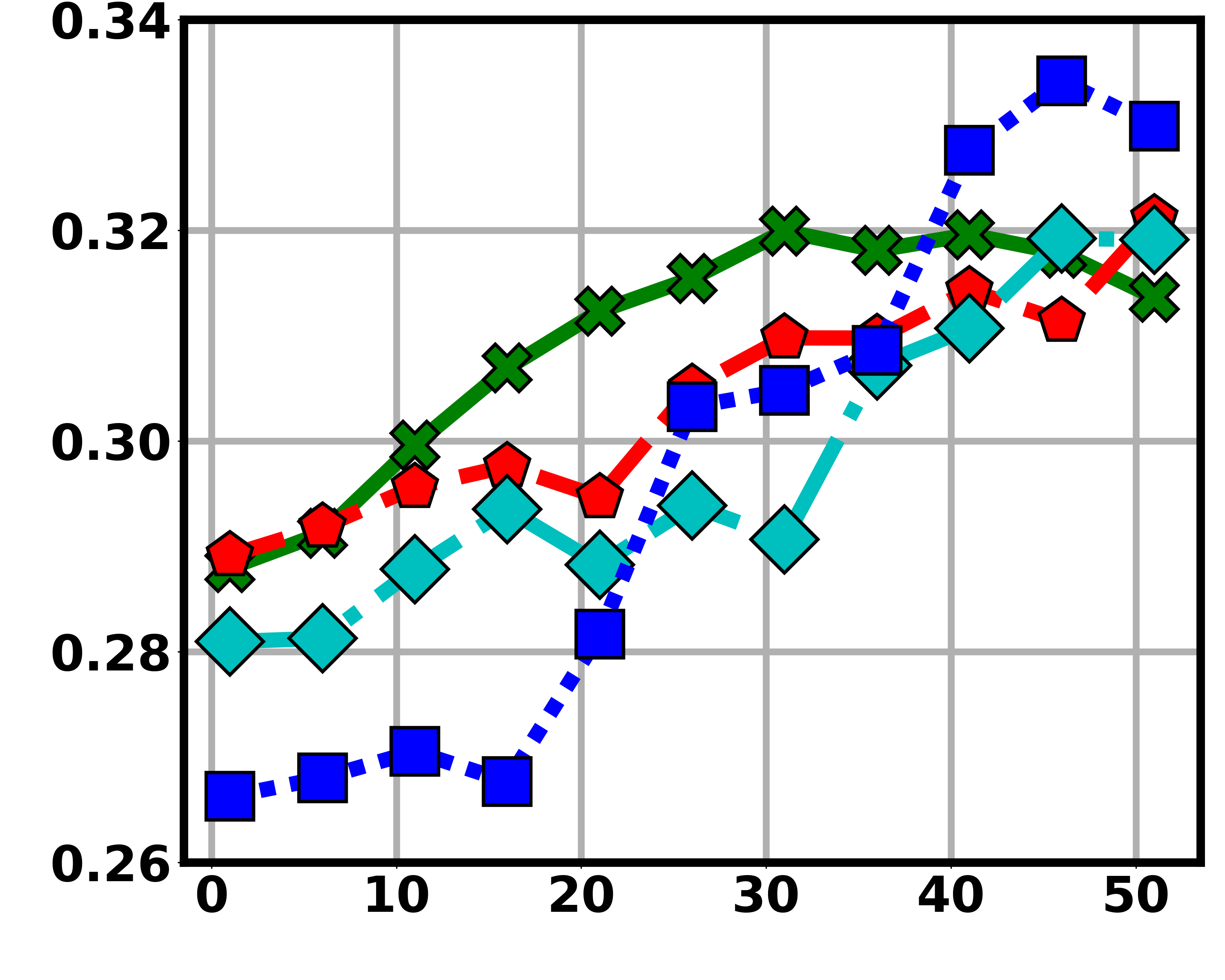}}

    \put(52, 10){\includegraphics[width=0.49\linewidth]{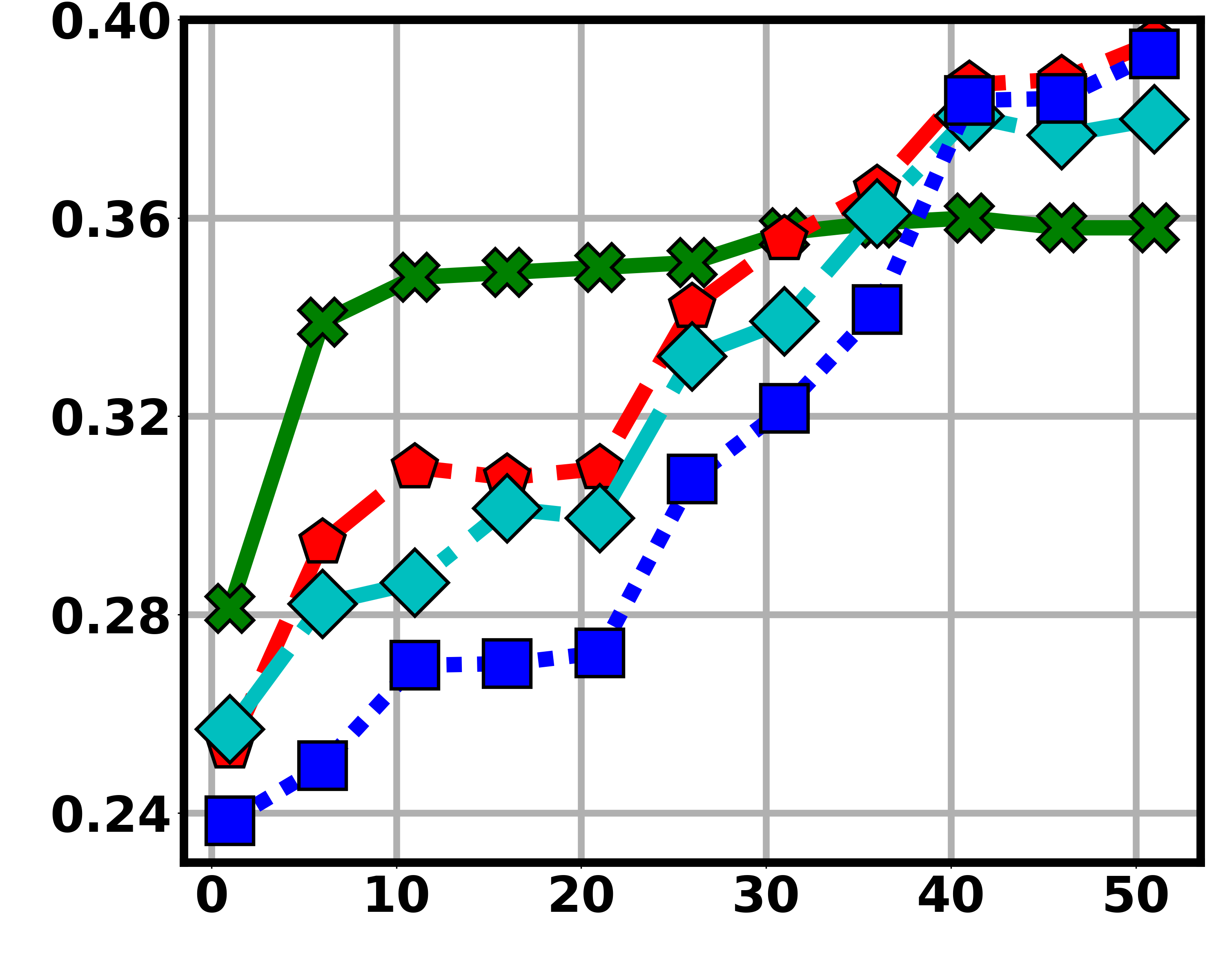}}
    \put(4.8, 10){\includegraphics[width=0.49\linewidth]{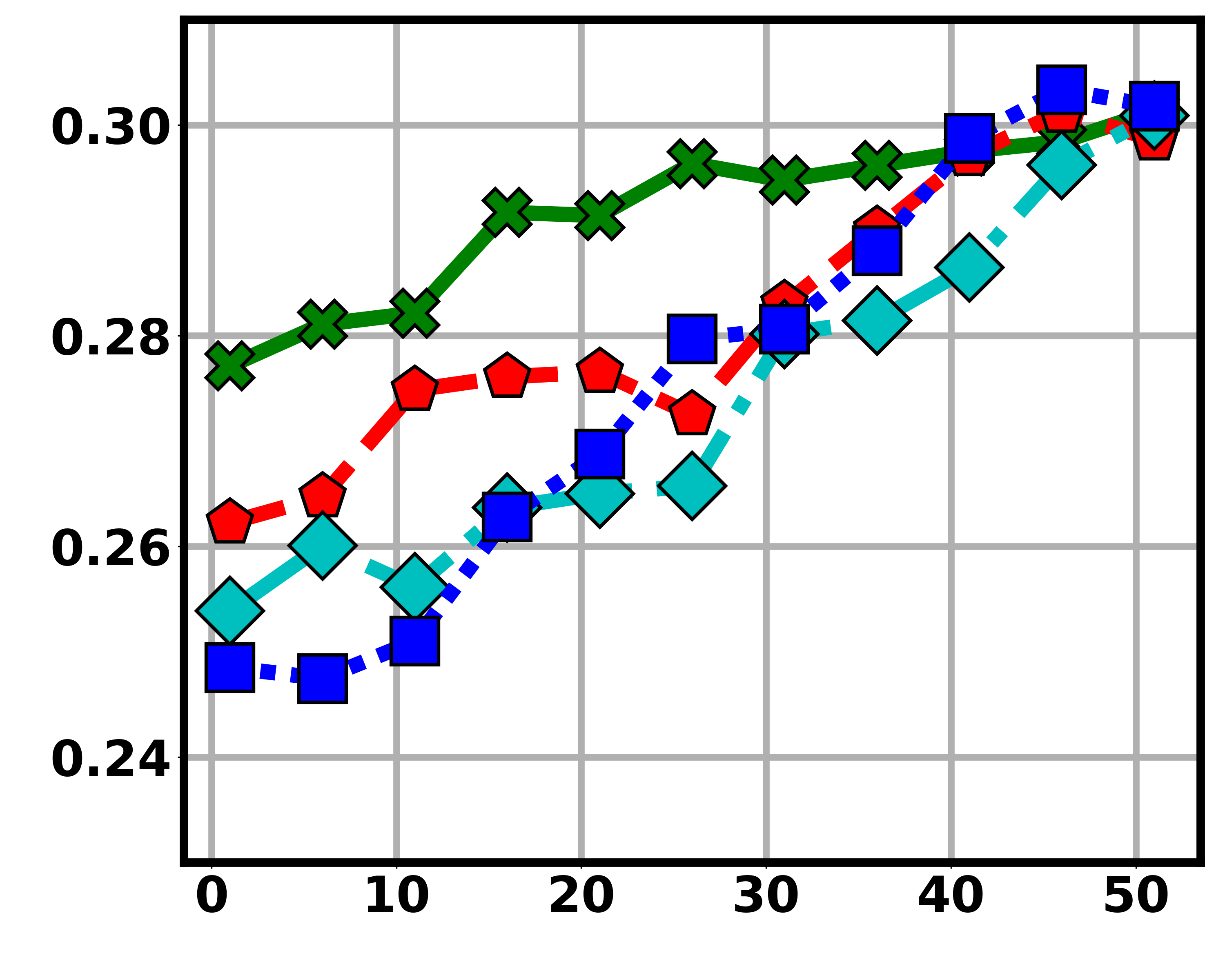}}

    \put(19, 10){\intab{\# of comm. rounds}}
    \put(66, 10){\intab{\# of comm. rounds}}

    \put(19, 52){\intab{\# of comm. rounds}}
    \put(66, 52){\intab{\# of comm. rounds}}

    \put(4.2, 24){\rotatebox{90}{\intab{ RMSE}}}
    \put(4.2, 66){\rotatebox{90}{\intab{ RMSE}}}

    \end{overpic}
    \caption{
        Reconstruction error versus the number of communication rounds with different DP noise levels. 
        During the first 20 communication rounds, stronger DP noise may reduce reconstruction error and amplify privacy risk.
        As training progresses, stronger DP noise can increase reconstruction error, thereby reducing privacy leakage.
        This phase transition in RMSE across different noise levels aligns with the insight from Conjecture C-2.
        \label{fig:claim2}}
\end{figure}
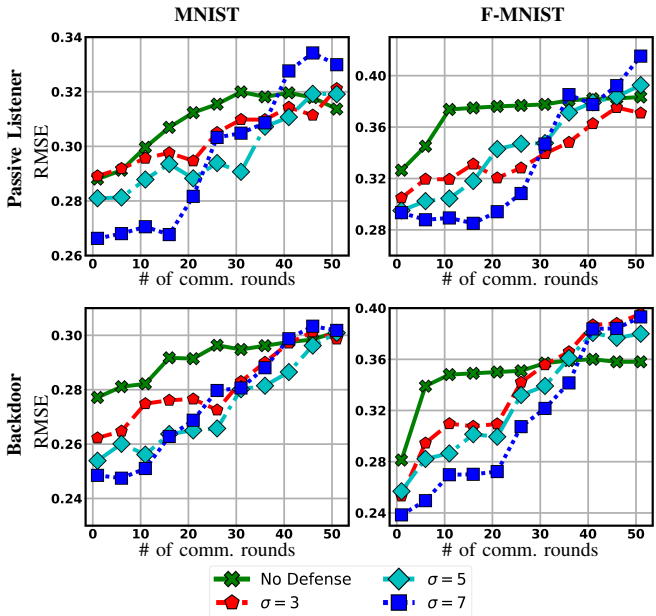

\begin{table}[!tbp]
    \centering 
    \caption{Test accuracy ($\%$) at different noise levels. }\label{tab:acc_c2}
    \begin{tabular}{p{0.1\linewidth}p{0.08\linewidth}p{0.11\linewidth}p{0.11\linewidth}p{0.11\linewidth}p{0.11\linewidth}}
        \toprule 
                & \intab[0.85]{Task}  & \intab[0.85]{No Defense} & \intab[0.85]{$\sigma=3$} & \intab[0.85]{$\sigma=5$} & \intab[0.85]{$\sigma=7$} \\ \midrule 
        \multirow{2}{0.1\linewidth}{\vspace*{-7pt} \intab[0.85]{Passive}\intab[0.85]{Listener}} & 
        \intab[0.85]{MNIST} & $80.9$ & $79.6$ & $78.3$ & $77.8$ \\  \cmidrule{2-6} 
        & \intab[0.85]{F-MNIST} & $ 73.4 $ & $65.1$ & $63.8$ & $59.2 $   \\ \midrule 
        \multirow{2}{0.1\linewidth}{\intab[0.85]{Backdoor}} & 
        \intab[0.85]{MNIST} & $  68.2 $ & $66.3$ & $64.7$ & $63.1 $ \\  \cmidrule{2-6} 
        & \intab[0.85]{F-MNIST} & $ 58.9 $ & $58.2$ & $31.0$ & $16.5 $  \\ \midrule 
         \multicolumn{2}{l}{\hspace*{-6pt} \intab[0.85]{Privacy Budget $\varepsilon$}} &  \intab[0.85]{NA} & $3.3$ & $1.7$ & $1.2$ \\ \bottomrule
    \end{tabular}
\end{table}
}

To validate Conjecture~C-2 and Remark~\ref{remark:impact_noise}, we evaluate how different local DP noise levels affect the reconstruction error. 
Specifically, we apply local DP against both a Passive Listener and a Backdoor-based attacker, using DP noise with standard deviations $\sigma \in \{3, 5, 7\}$.
Figure~\ref{fig:claim2} shows that a higher noise level does not consistently improve privacy. 
Specifically, during the first $20$ communication rounds, stronger DP noise surprisingly leads to lower reconstruction error, indicating that the attacker is more successful in reconstructing private data compared to the No Defense baselines. 
This observation supports Remark~\ref{remark:impact_noise} and Conjecture~C-2. 
As training progresses, the effect of DP noise becomes more complex. 
For example, while $\sigma=7$ initially enables better reconstruction, it eventually results in noisier gradients that hinder the successful reconstruction. 
Within the context of the maliciously curious client threat model, the protection provided by DP is not monotonic and can vary depending on the training progress.
We emphasize that local DP may inadvertently interact with the maliciously curious client's behavior to boost the attack. This is reflected in the lower reconstruction error achieved by the Backdoor poisoning case (bottom row), particularly at $\sigma = 7$ on the MNIST task, compared to the Passive Listener (top row).

TABLE~\ref{tab:acc_c2} reports the test accuracy at different noise levels at the end of the training. 
For the privacy budget $\epsilon$, we consider $\delta=10^{-5}$, estimated using the central limit theorem~\cite{dong2022gaussian}. 
The results in TABLE~\ref{tab:acc_c2} show that stronger DP noise generally leads to reduced test accuracy and corresponds to a smaller privacy budget, aligning with the standard utility-privacy trade-off in DP analysis. 
\afterpage{%
\begin{figure}[!tb]
    \begin{overpic}[width=\linewidth, height=1.4\linewidth]{fig/4x4.pdf}
    
    \put(16.5, 2){\includegraphics[width=0.62\linewidth]{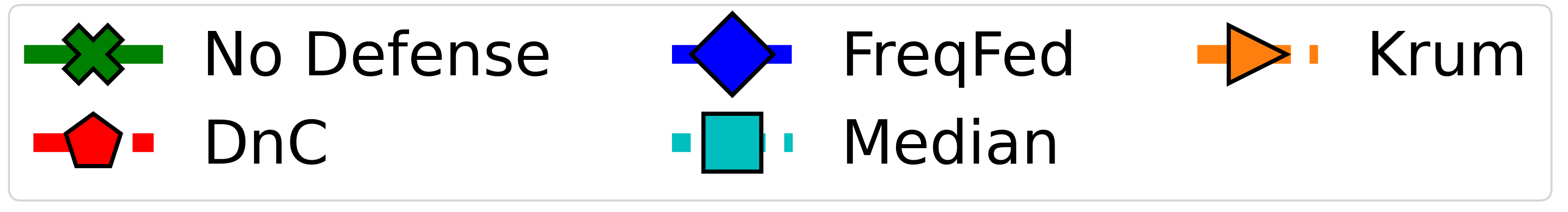}}
    
    \put(0, 77){\rotatebox{90}{\intab{\bf Passive Listener}}}
    \put(0, 49){\rotatebox{90}{\intab{\bf Backdoor}}}
    \put(0, 17){\rotatebox{90}{\intab{\bf Sign Flipping}}}

    \put(18, 99.3){\intab{\bf MNIST }}
    \put(50, 99.3){\intab{\bf F-MNIST }}

    \put(37, 10){\includegraphics[width=0.49\linewidth]{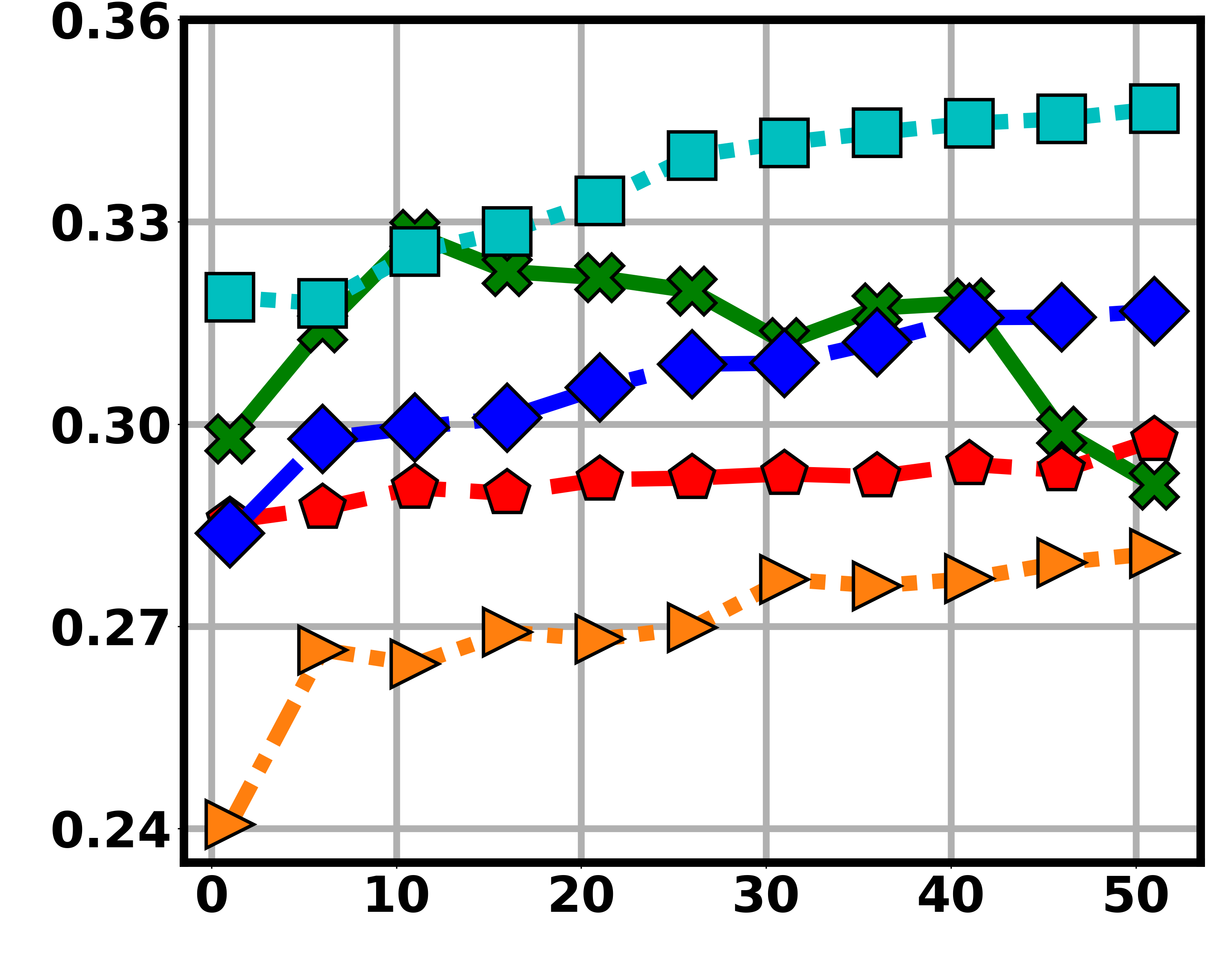}}
    \put(3.2, 10){\includegraphics[width=0.49\linewidth]{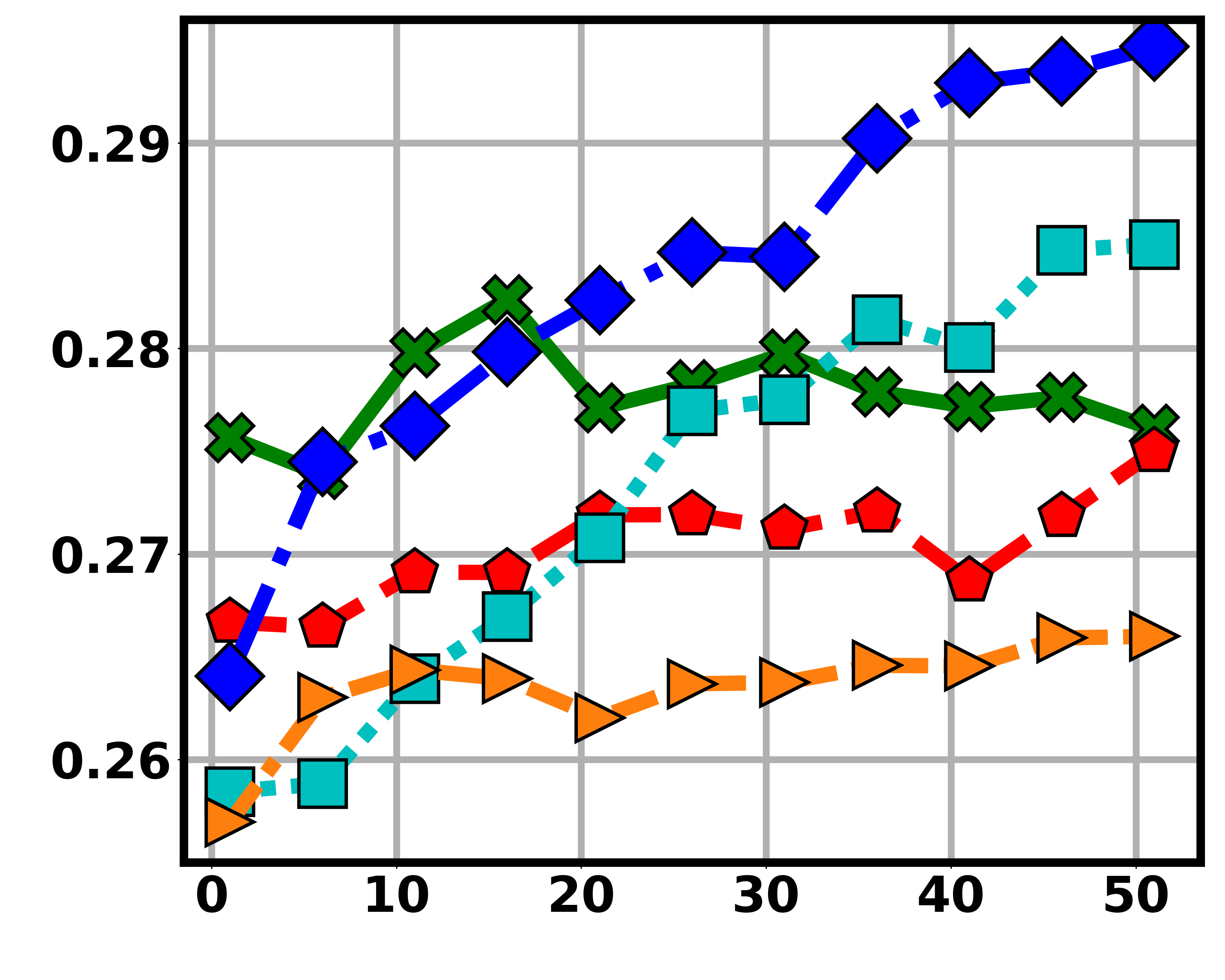}}

    \put(37, 40){\includegraphics[width=0.49\linewidth]{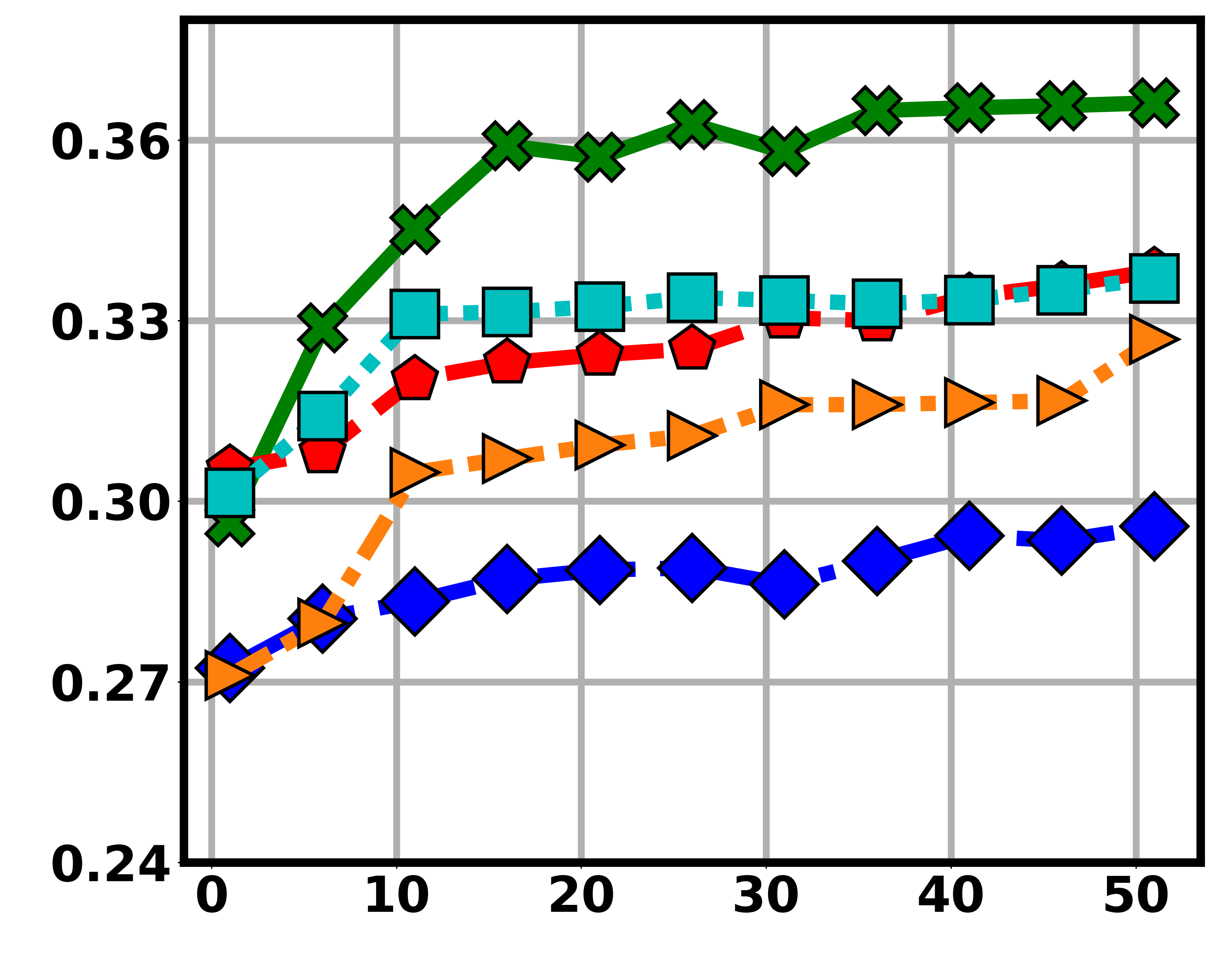}}
    \put(3.2, 40){\includegraphics[width=0.49\linewidth]{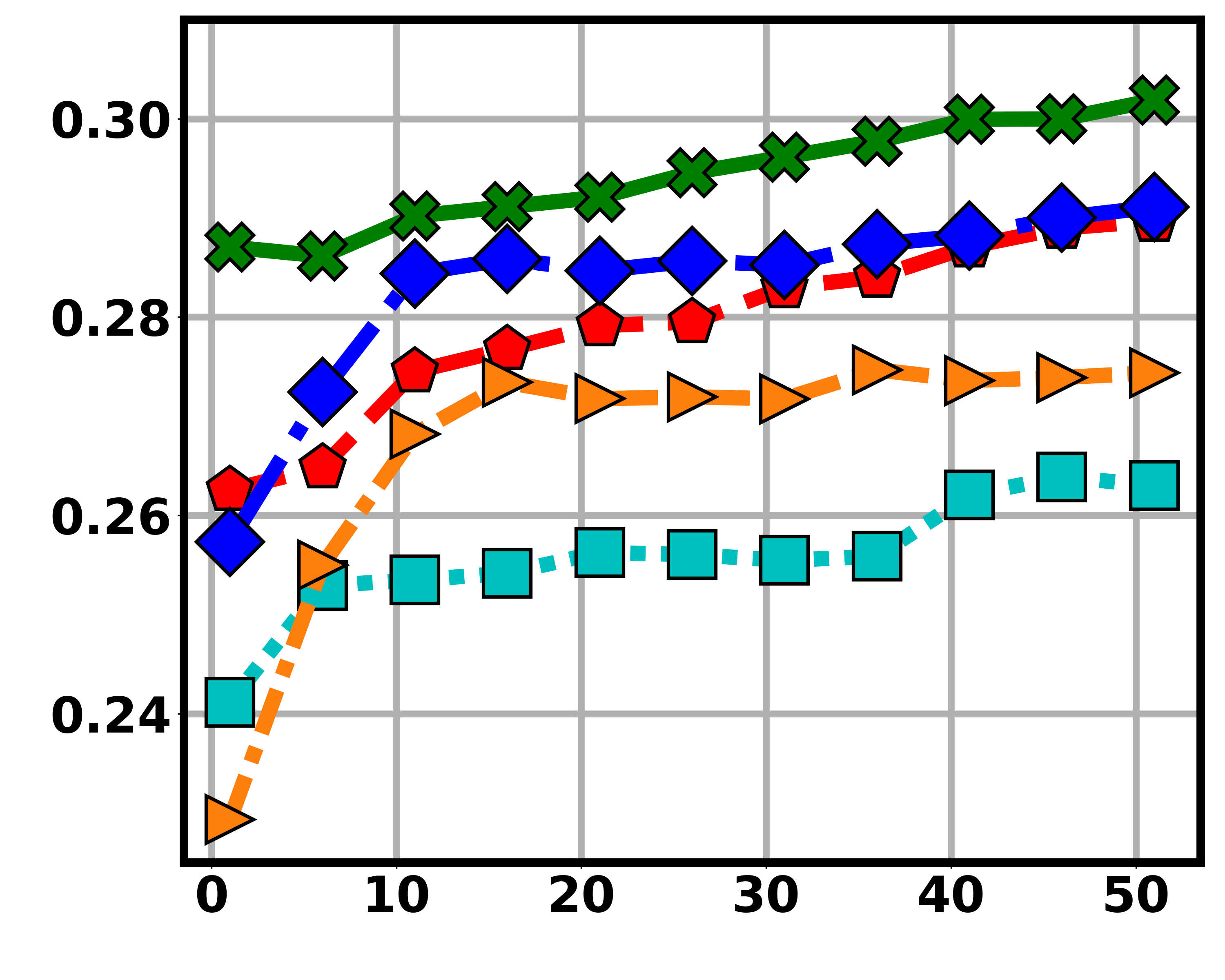}}

    \put(37, 70){\includegraphics[width=0.49\linewidth]{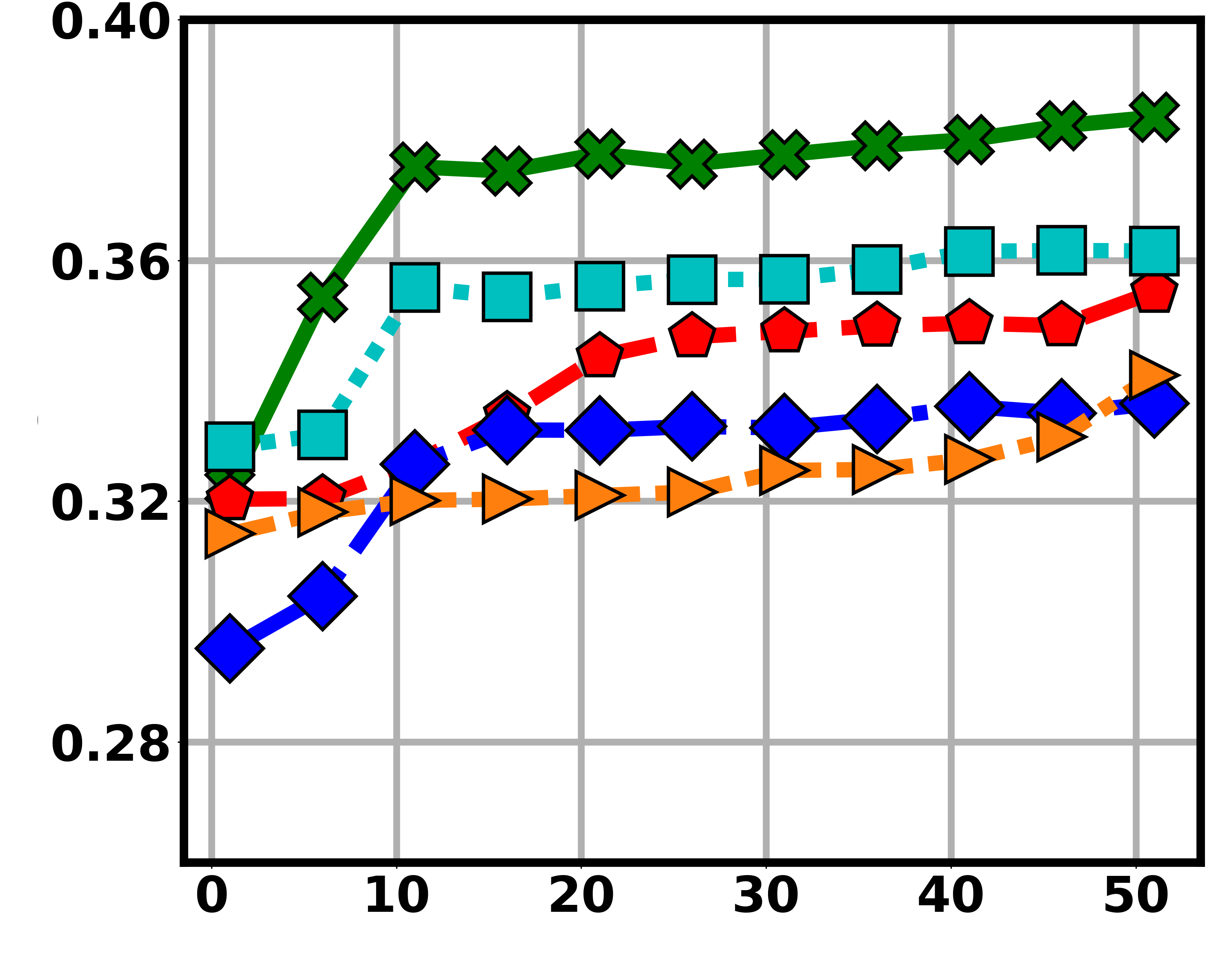}}
    \put(3.2, 70){\includegraphics[width=0.49\linewidth]{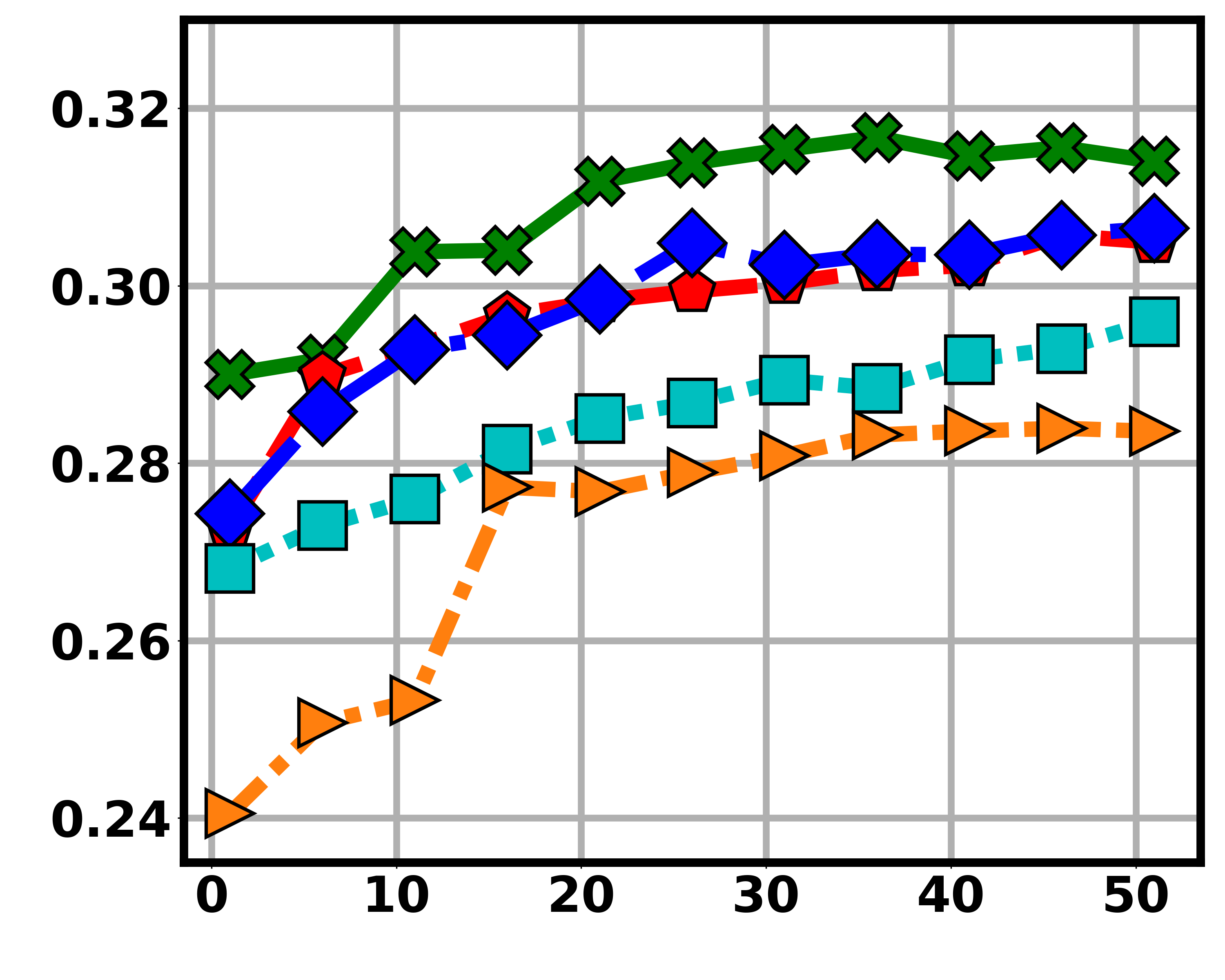}}
    
    \put(13, 10){\intab{\# of comm. rounds}}
    \put(46.8, 10){\intab{\# of comm. rounds}}

    \put(13, 40){\intab{\# of comm. rounds}}
    \put(46.8, 40){\intab{\# of comm. rounds}}

    \put(13, 70){\intab{\# of comm. rounds}}
    \put(46.8, 70){\intab{\# of comm. rounds}}

    \put(2.5, 20){\rotatebox{90}{\intab{ RMSE}}}
    \put(2.5, 50){\rotatebox{90}{\intab{ RMSE}}}
    \put(2.5, 80){\rotatebox{90}{\intab{ RMSE}}}

    \end{overpic}
    \vspace*{-10pt}
    \caption{
    Reconstruction error versus the number of communication rounds under different Byzantine-robust defenses across poisoning functions, including Passive Listener (top), Backdoor Attack (middle), and Sign Flipping (bottom). 
    For the Passive Listener and Backdoor attacks, adopting Byzantine-robust defenses may unintentionally ease the reconstruction attack compared to the ``No Defense'' baseline.
    This observation is consistent with Conjecture C-3. 
    \label{fig:claim3}}
\end{figure}
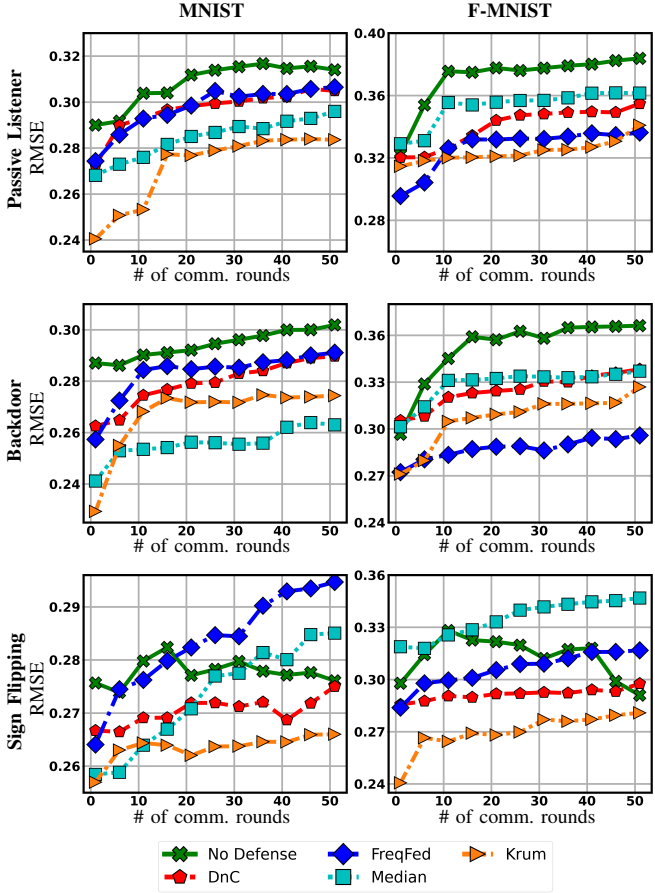

\begin{table}[!tbp]
    \centering 
    \caption{Test accuracy ($\%$)  under different robust aggregation methods. }\label{tab:acc_c3}
    \begin{tabular}{p{0.09\linewidth}p{0.08\linewidth}p{0.09\linewidth}p{0.09\linewidth}p{0.09\linewidth}p{0.09\linewidth}p{0.09\linewidth}}
        \toprule 
        & \intab[0.85]{Task}  & \intab[0.8]{No Defense} & \intab[0.85]{DnC} & \intab[0.85]{FreqFed} & \intab[0.85]{Median} & \intab[0.85]{Krum} \\ \midrule 
        \multirow{2}{0.1\linewidth}{\vspace*{-7pt} \intab[0.85]{Passive} \intab[0.85]{Listener}} & 
        \intab[0.85]{MNIST} & $82.0$ & $69.8$ & $74.7$ & $63.3$ & $59.2$ \\  \cmidrule{2-7} 
        & \intab[0.85]{F-MNIST} & $ 72.7 $ & $59.6$ & $55.4$ & $60.5$ & $47.9 $   \\ \midrule 
        \multirow{2}{0.1\linewidth}{\intab[0.85]{Backdoor}} & 
        \intab[0.85]{MNIST} & $  69.4 $ & $72.3$ & $72.5$ & $59.0$ & $57.7 $ \\  \cmidrule{2-7} 
        & \intab[0.85]{F-MNIST} & $ 59.6 $ & $61.2$ & $63.1$ & $47.9 $ & $32.3$ \\ \midrule 
        \multirow{2}{0.1\linewidth}{\vspace*{-8pt} \intab[0.85]{Sign} \intab[0.85]{Flipping} \intab[0.85]{}} & 
        \intab[0.85]{MNIST} & $  49.5 $ & $58.0$ & $67.2$ & $62.3$ & $53.2 $ \\  \cmidrule{2-7} 
        & \intab[0.85]{F-MNIST} & $ 45.6 $ & $60.9$ & $51.2$ & $60.4 $ & $43.1$  \\ \bottomrule 
        \end{tabular}
\end{table}
}

However, Figure~\ref{fig:claim2} highlights a different interplay between utility and DP noise level, where increasing the noise strength with the intention to enhance privacy may lead to lower reconstruction error and thereby amplify privacy leakage.


\subsection{Effectiveness of Server-Side Robust Defenses}\label{sec:byzantine}

To validate Conjecture~C-3, we evaluate how robust aggregation defenses may affect privacy in both directions. 
Experiments are conducted across three types of maliciously curious client attacks, including Passive Listener, Backdoor Attack, and Sign Flipping. 
Backdoor Attack does not directly manipulate the gradients and can be considered a more stealthy poisoning compared to Sign Flipping. 
Within each subplot of Figure~\ref{fig:claim3}, we compare the reconstruction error across different Byzantine-robust defenses, including Krum, Median, DnC, and FreqFed, as well as the No Defense baseline.

In the top row with the Passive Listener attack scenario, it can be observed that adopting Byzantine robust defenses unintentionally ease the reconstruction attack compared to the baseline, which validates Conjecture~C-3.
A similar phenomenon can be observed from the middle row with Backdoor poisoning. 
In the bottom row, when defenses effectively suppress the attack, such as Median against Sign Flipping on the F-MNIST task in which the reconstruction error is increased, indicating improved privacy protection.
We caution that, however, misaligned defenses can backfire. 
For example, adopting DnC defense against the Sign Flipping attack may result in an easier reconstruction attack compared to the No Defense baseline, essentially amplifying the privacy leakage. 
In practice, the attacker may switch the poisoning function in different communication rounds to further increase the efficacy of its attack.

We additionally show the final test accuracy in TABLE~\ref{tab:acc_c3}. 
For the Passive Listener attack, Byzantine-robust defenses result in lower test accuracy compared to the No Defense baseline. 
This reflects that improving robustness often comes at the cost of reduced model accuracy. 
Notably, even when such defenses as DnC restore model accuracy from the No Defense baseline to a certain extent, the reconstruction attack could be more successful than the No Defense baseline (see the bottom row of Figure~\ref{fig:claim3}).
These observations illustrate the distinct interplay between robustness and privacy under the proposed maliciously curious client threat model, supporting Conjecture~C-3. 
Applying Byzantine-robust defenses with the intention to enhance robustness may sometimes amplify privacy leakage.


\subsection{Combined Defenses}\label{sec:combined_defense}

\afterpage{%
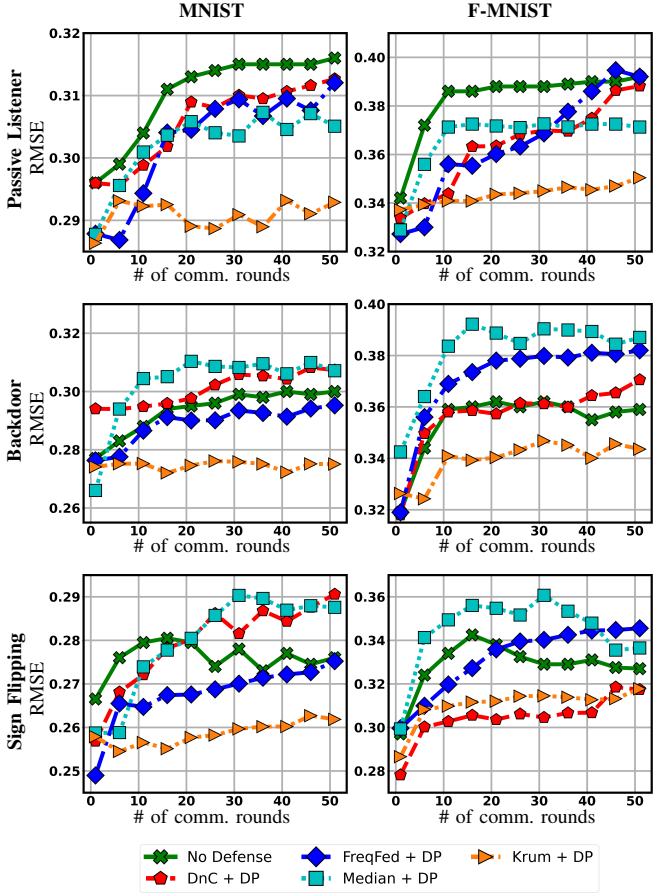
\begin{figure}[!tb]
    \begin{overpic}[width=\linewidth, height=1.4\linewidth]{fig/4x4.pdf}
    
    \put(14.5, 2){\includegraphics[width=0.72\linewidth]{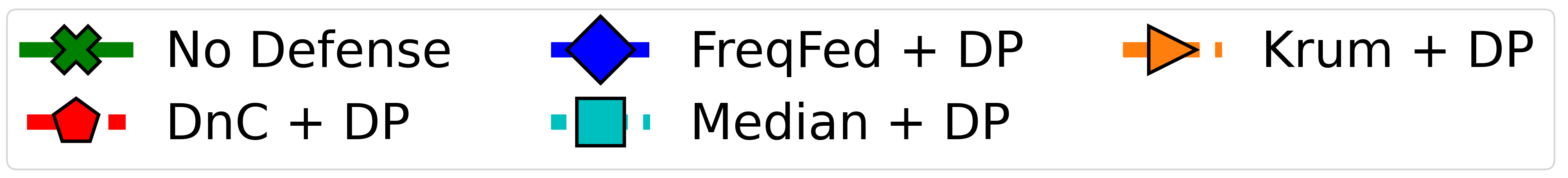}}
    
    \put(0, 77){\rotatebox{90}{\intab{\bf Passive Listener}}}
    \put(0, 49){\rotatebox{90}{\intab{\bf Backdoor}}}
    \put(0, 17){\rotatebox{90}{\intab{\bf Sign Flipping}}}

    \put(18, 99.3){\intab{\bf MNIST }}
    \put(50, 99.3){\intab{\bf F-MNIST }}

    \put(37, 10){\includegraphics[width=0.49\linewidth]{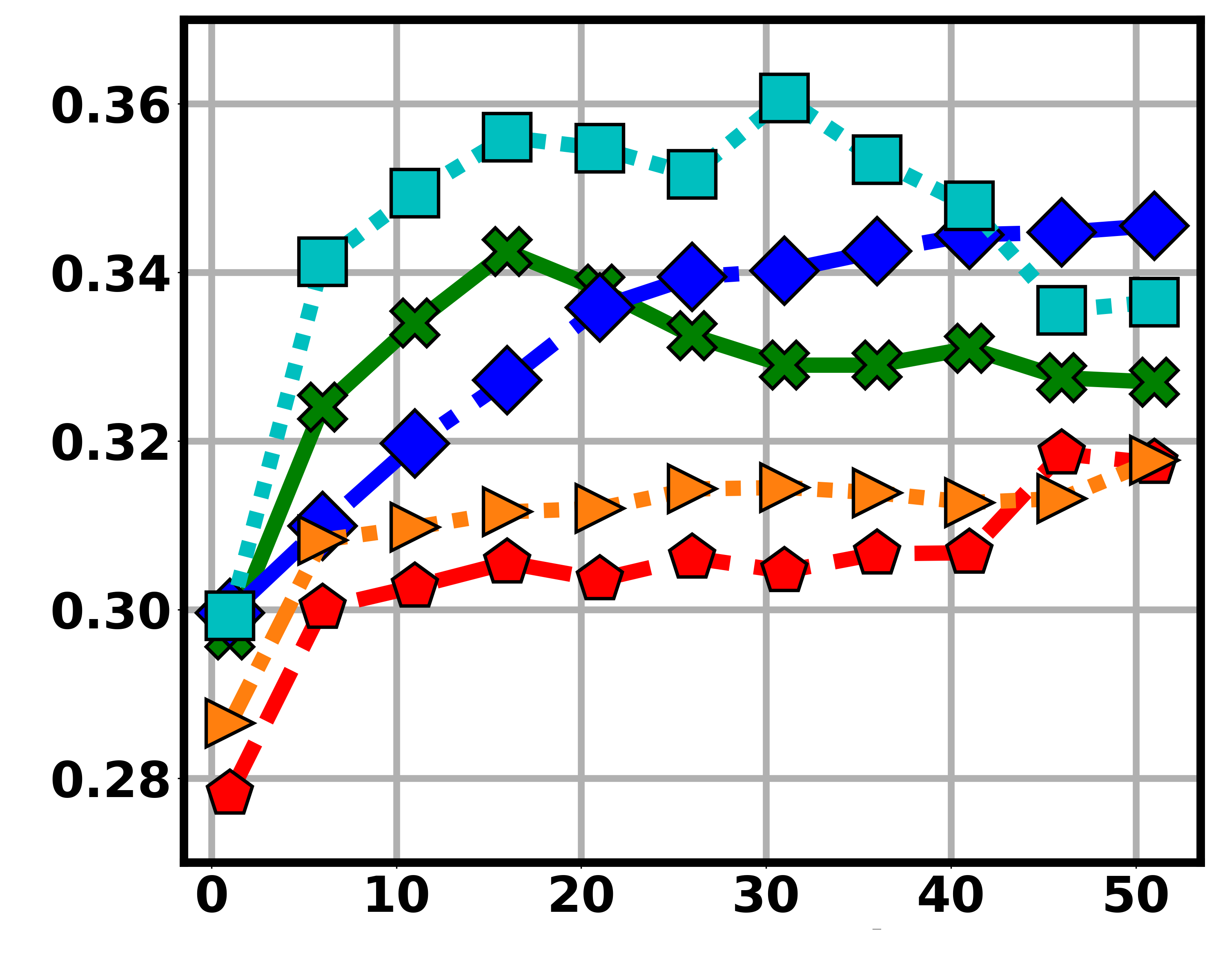}}
    \put(3.2, 10){\includegraphics[width=0.49\linewidth]{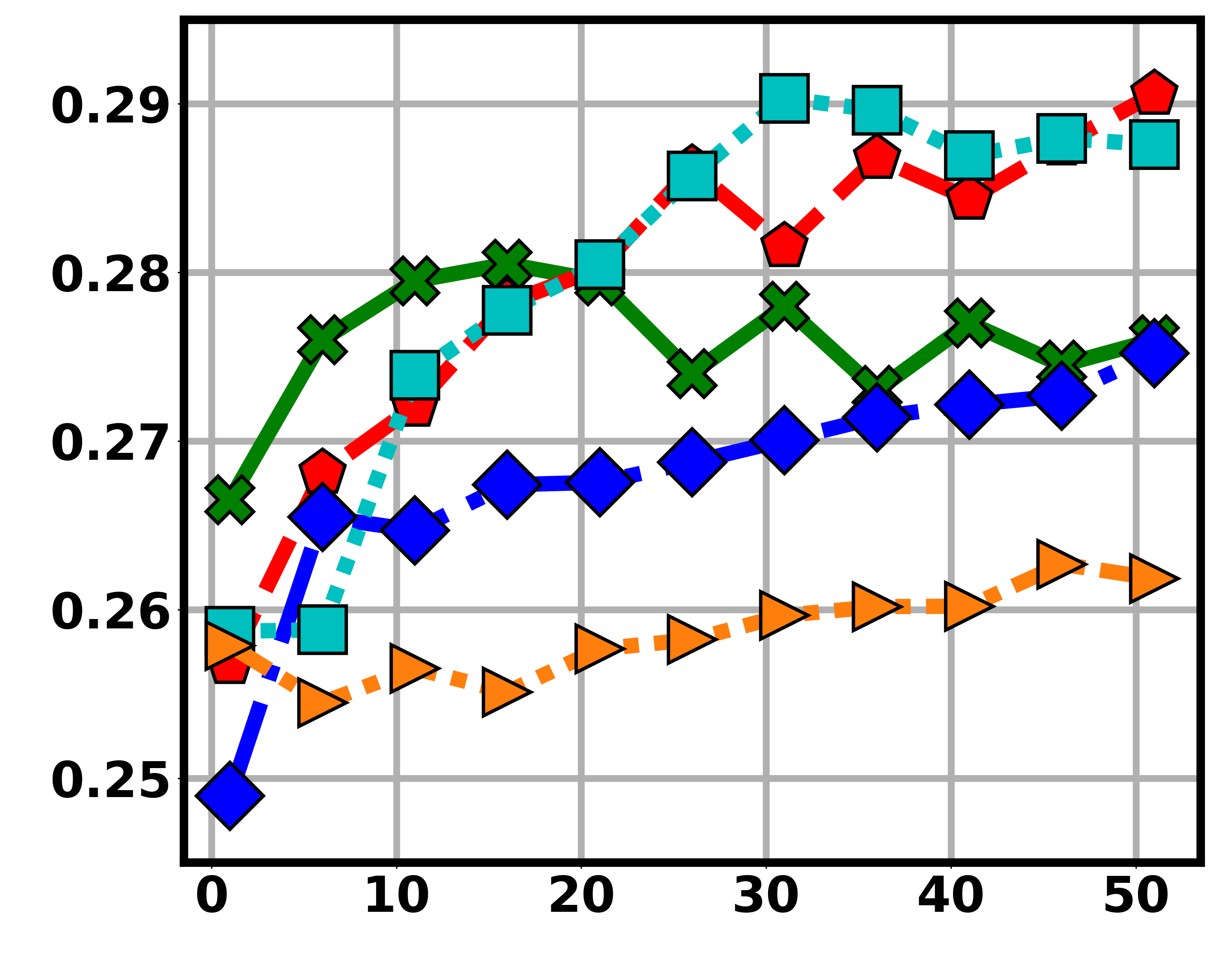}}

    \put(37, 40){\includegraphics[width=0.49\linewidth]{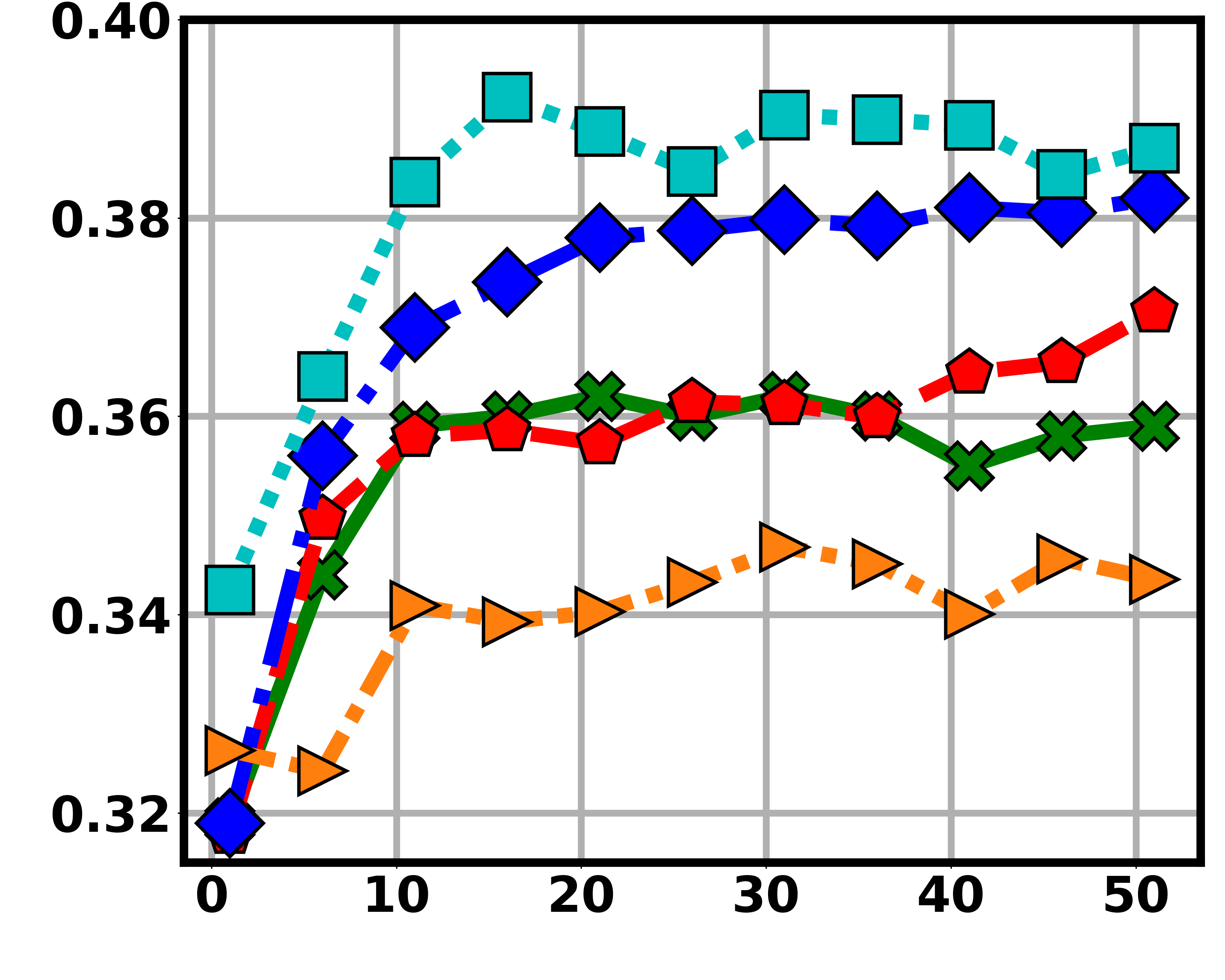}}
    \put(3.2, 40){\includegraphics[width=0.49\linewidth]{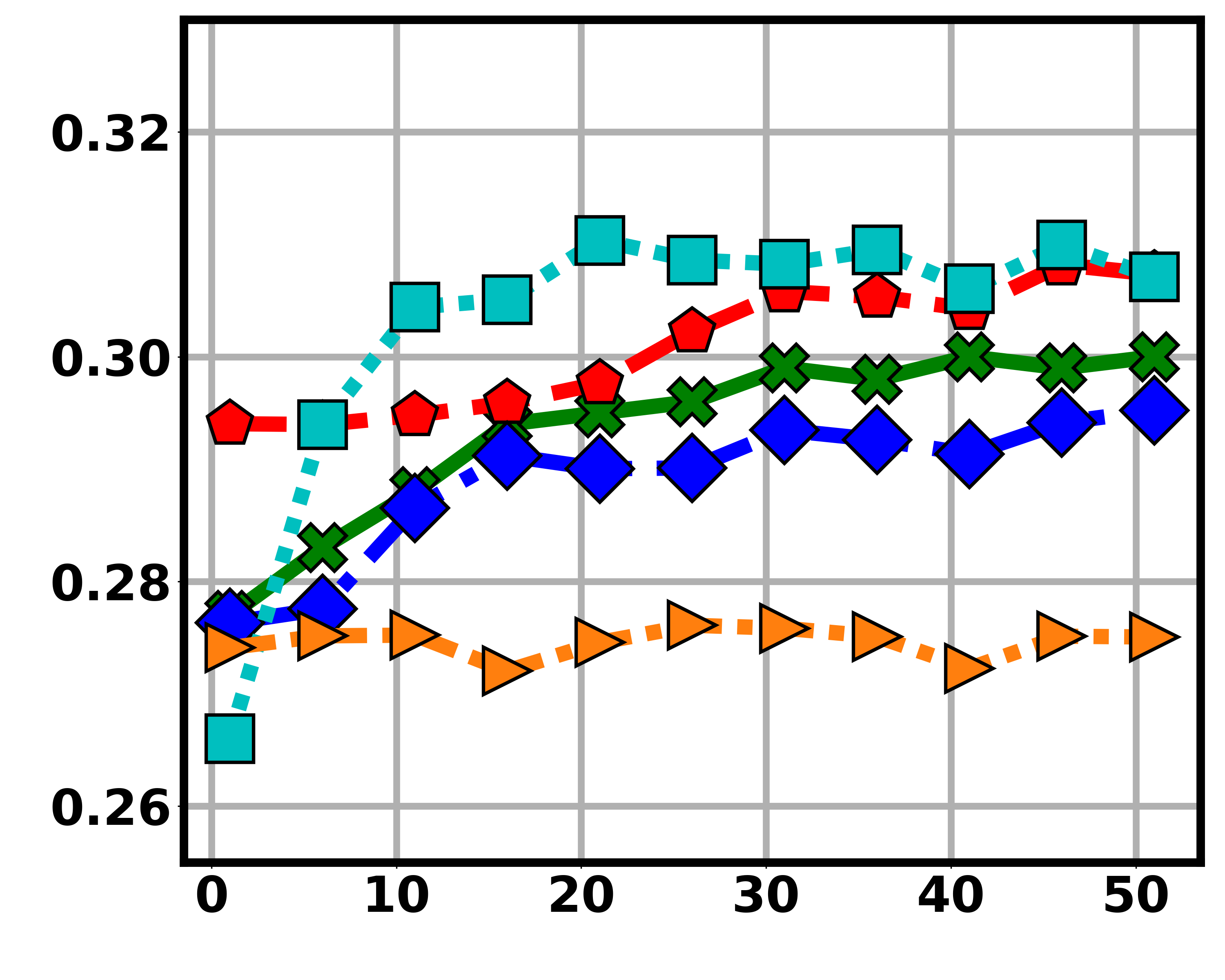}}

    \put(37, 70){\includegraphics[width=0.49\linewidth]{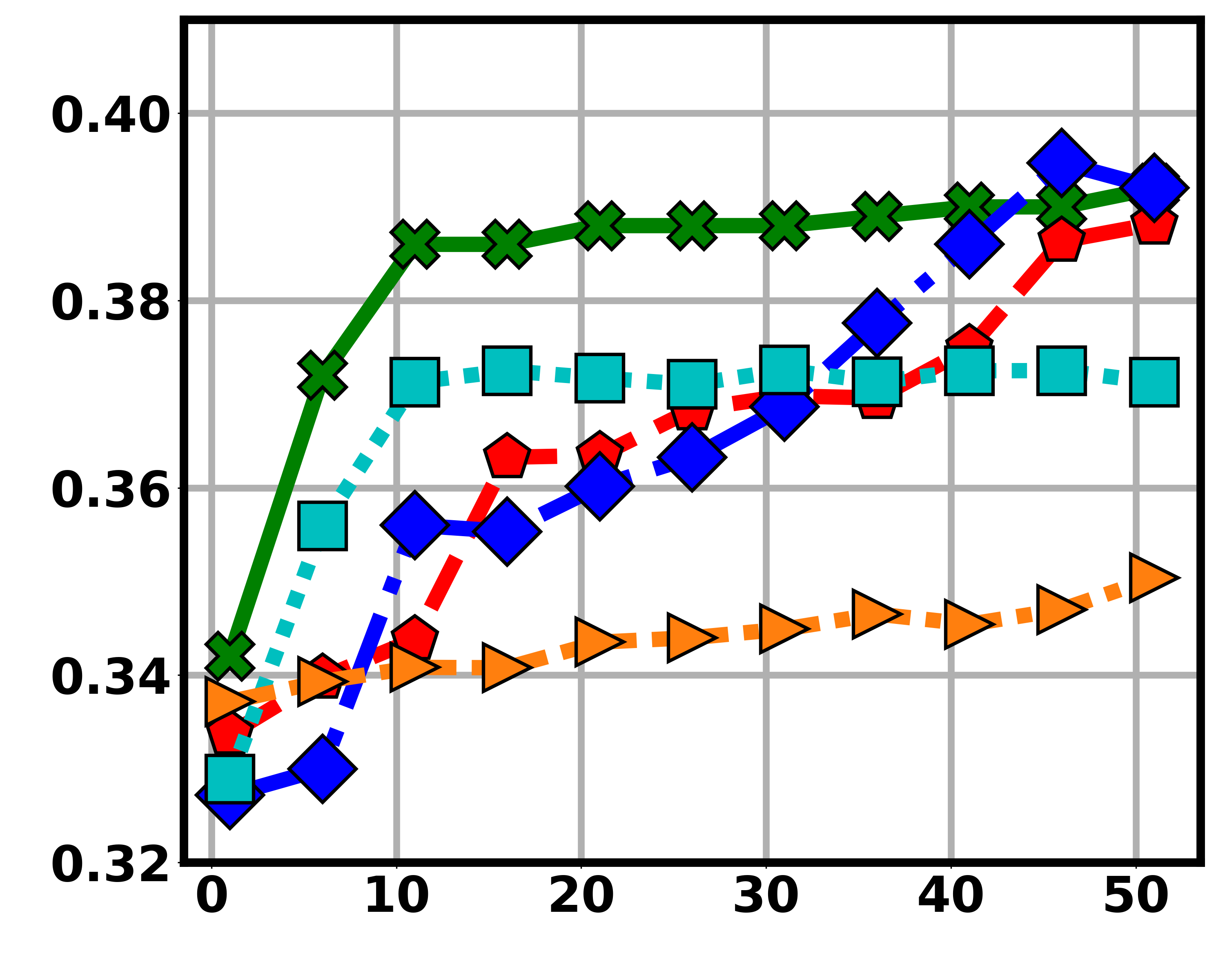}}
    \put(3.2, 70){\includegraphics[width=0.49\linewidth]{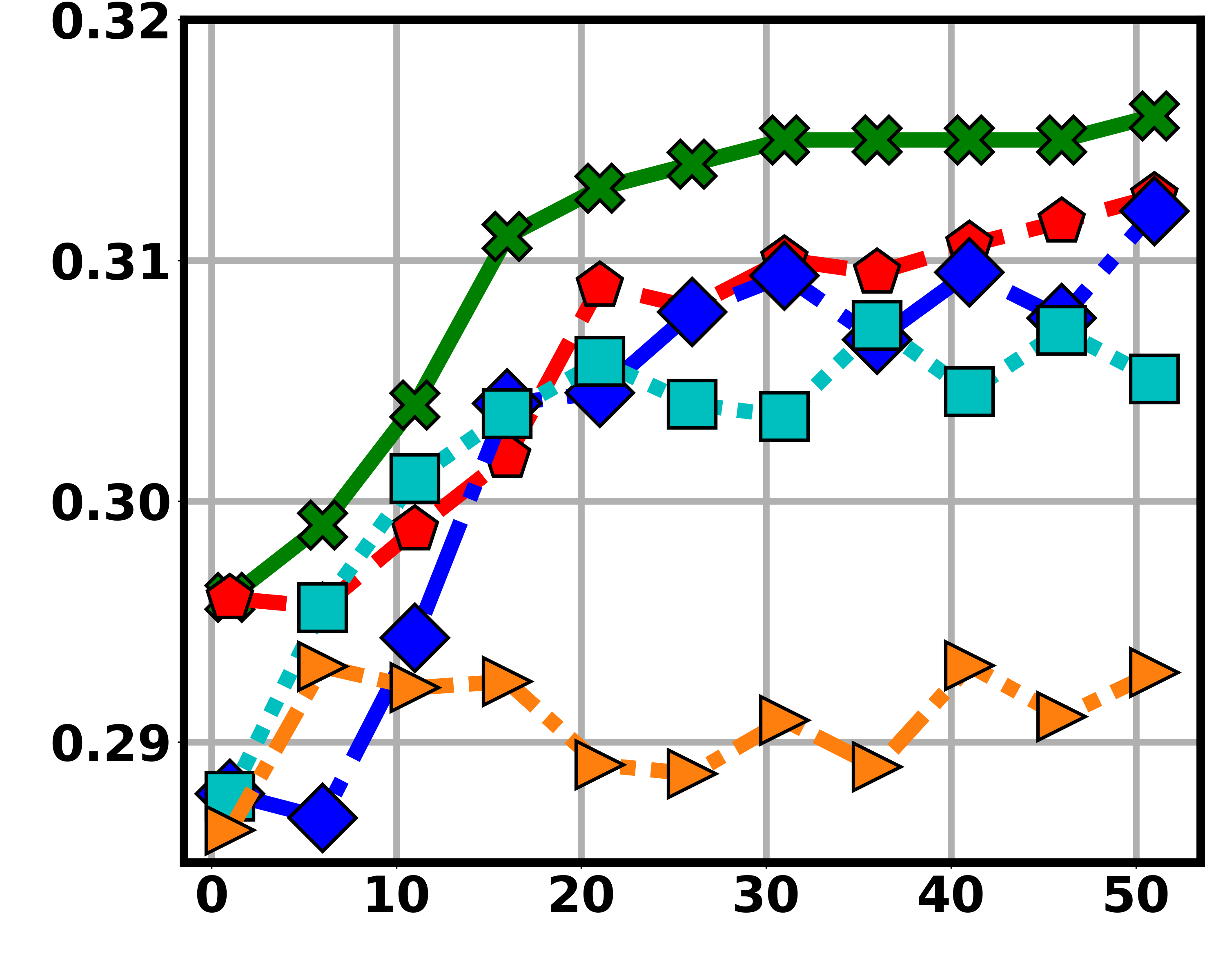}}
    
    \put(13, 10){\intab{\# of comm. rounds}}
    \put(46.8, 10){\intab{\# of comm. rounds}}

    \put(13, 40){\intab{\# of comm. rounds}}
    \put(46.8, 40){\intab{\# of comm. rounds}}

    \put(13, 70){\intab{\# of comm. rounds}}
    \put(46.8, 70){\intab{\# of comm. rounds}}

    \put(2.5, 20){\rotatebox{90}{\intab{ RMSE}}}
    \put(2.5, 50){\rotatebox{90}{\intab{ RMSE}}}
    \put(2.5, 80){\rotatebox{90}{\intab{ RMSE}}}

    \end{overpic}
    \vspace*{-10pt}
    \caption{
    Reconstruction error versus the number of communication rounds under combined defenses. 
    For the Passive Listener, adopting combined defenses may reduce the RMSE achieved by the non-defense server.
    Under Backdoor and Sign Flipping, a mistakenly used combination (e.g., Krum + DP) may amplify privacy leakage compared to the No Defense baseline.
    \label{fig:claim4}}
\end{figure}

\begin{table}[!tbp]
    \centering 
    \caption{Test accuracy ($\%$)  under different combined defenses. }\label{tab:acc_c4}
    \begin{tabular}{p{0.09\linewidth}p{0.08\linewidth}p{0.09\linewidth}p{0.09\linewidth}p{0.09\linewidth}p{0.09\linewidth}p{0.09\linewidth}}
        \toprule 
        & \intab[0.85]{\vspace{10pt} Task}  & \intab[0.85]{No} \intab[0.85]{Defense} & \intab[0.85]{DnC} \intab[0.85]{+DP} & \intab[0.85]{FreqFed } \intab[0.85]{+DP} & \intab[0.85]{Median } \intab[0.85]{+DP} & \intab[0.85]{Krum } \intab[0.85]{+DP} \\ \midrule 
        \multirow{2}{0.1\linewidth}{\vspace*{-7pt} \intab[0.85]{Passive} \intab[0.85]{Listener}} & 
        \intab[0.85]{MNIST} & $81.7$ & $70.1$ & $71.9$ & $60.2$ & $49.3$ \\  \cmidrule{2-7} 
        & \intab[0.85]{F-MNIST} & $ 73.4 $ & $56.4$ & $53.0$ & $63.5$ & $59.2 $   \\ \midrule 
        \multirow{2}{0.1\linewidth}{\intab[0.85]{Backdoor}} & 
        \intab[0.85]{MNIST} & $  67.3 $ & $59.2$ & $68.5$ & $58.6$ & $40.7 $ \\  \cmidrule{2-7} 
        & \intab[0.85]{F-MNIST} & $ 58.9 $ & $52.8$ & $63.1$ & $50.4 $ & $17.6$ \\ \midrule 
        \multirow{2}{0.1\linewidth}{\vspace*{-8pt} \intab[0.85]{Sign} \intab[0.85]{Flipping} \intab[0.85]{}} & 
        \intab[0.85]{MNIST} & $  48.5 $ & $68.0$ & $64.7$ & $62.3$ & $53.2 $ \\  \cmidrule{2-7} 
        & \intab[0.85]{F-MNIST} & $ 43.9 $ & $59.6$ & $52.4$ & $58.5 $ & $31.8$  \\ \bottomrule 
        \end{tabular}
\end{table}
}


    




We now investigate whether combining client-side with server-side defenses can more effectively protect against data reconstruction. 
We set $\sigma = 3$ in client-side local DP and apply the server-side robust defenses evaluated in Section~\ref{sec:byzantine}. 
Figure~\ref{fig:claim4} shows that, in some cases, combining local DP with Byzantine-resilient defenses increases the reconstruction error compared to using Byzantine defenses alone, indicating improved privacy. 
For example, comparing the bottom row of Figure~\ref{fig:claim3} with the bottom row of Figure~\ref{fig:claim4} on F-MNIST, we observe that adding local DP increases the reconstruction error under DnC and Median aggregation, confirming DP's privacy benefit in certain scenarios.
However, this improvement is not consistent. 
In the top row of the Passive Listener attack scenario, the combined defense still yields lower reconstruction error than the No Defense baseline.
Moreover, in the second and bottom rows, combining Krum with DP consistently yields lower reconstruction error than the No Defense baseline.
These results highlight the complex interaction between privacy and robustness in FL.
To address this challenge, we recommend exploring advanced protocols such as stochastic sign-based SGD~\cite{jin2020stochastic, li2022communication}, comprehensively integrating quantization, local DP, and robust aggregation into the FL framework.
We leave the investigation of these advanced methods for future work.

We also report the final test accuracy in TABLE~\ref{tab:acc_c4} to reflect the impact of combining local DP with Byzantine defenses on model performance. 
For the Passive Listener attack scenario, the combined defenses generally reduce test accuracy compared to the No Defense baseline, reflecting the trade-offs introduced by robustness and local DP. 
For the Sign Flipping attack, certain combinations, such as FreqFed + DP on the MNIST task, may improve accuracy relative to the No Defense baseline. 
However, this does not necessarily translate to stronger privacy protection, as the reconstruction error in Figure~\ref{fig:claim4} remains lower, indicating more severe privacy leakage.


\subsection{Impact of Initial Conditions}

\afterpage{

\begin{table}[!tbp]
    \centering 
    \caption{Reconstruction error (RMSE, $\times 10^{-1}$) with different initialization methods. }\label{tab:reconstruction error_s4}
    \begin{tabular}{p{0.1\linewidth}p{0.1\linewidth}p{0.16\linewidth}p{0.16\linewidth}p{0.16\linewidth}}
        \toprule 
                & \intab[0.85]{Task}  & \intab[0.85]{Kaiming~\cite{he2015delving}} & \intab[0.85]{LeCun~\cite{lecun2002efficient}} & \intab[0.85]{Orthogonal~\cite{saxe2013exact}} \\ \midrule 
        \multirow{2}{0.1\linewidth}{\intab[0.85]{Passive}} & 
        \intab[0.85]{MNIST} & $2.91$ & $2.85$ & $2.77$ \\  \cmidrule{2-5} 
        & \intab[0.85]{F-MNIST} & $ 3.92 $ & $3.52$ & $3.36$     \\ \midrule 
        \multirow{2}{0.1\linewidth}{\intab[0.85]{Backdoor}} & 
        \intab[0.85]{MNIST} & $  2.78 $ & $2.65$ & $2.47$  \\  \cmidrule{2-5} 
        & \intab[0.85]{F-MNIST} & $ 3.88 $ & $3.35$ & $3.23$ \\ \midrule 
    \end{tabular}
\end{table}

\begin{figure}[!tb]
    \begin{overpic}[width=\linewidth, height=0.6\linewidth]{fig/4x4.pdf}
        \put(18, 40){\includegraphics[width=0.1\linewidth]{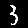}}
        \put(30, 40){\includegraphics[width=0.1\linewidth]{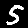}}
        \put(18, 28.5){\includegraphics[width=0.1\linewidth]{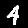}}
        \put(30, 28.5){\includegraphics[width=0.1\linewidth]{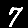}}
        
        \put(48, 40){\includegraphics[width=0.1\linewidth]{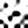}}
        \put(60, 40){\includegraphics[width=0.1\linewidth]{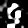}}
        \put(48, 28.5){\includegraphics[width=0.1\linewidth]{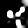}}
        \put(60, 28.5){\includegraphics[width=0.1\linewidth]{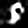}}
    
        \put(78, 40){\includegraphics[width=0.1\linewidth]{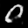}}
        \put(90, 40){\includegraphics[width=0.1\linewidth]{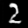}}
        \put(78, 28.5){\includegraphics[width=0.1\linewidth]{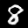}}
        \put(90, 28.5){\includegraphics[width=0.1\linewidth]{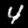}}

        \put(0, 38){\intab[0.85]{MNIST}}


        \put(18, 14){\includegraphics[width=0.1\linewidth]{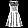}}
        \put(30, 14){\includegraphics[width=0.1\linewidth]{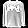}}
        \put(18, 2.5){\includegraphics[width=0.1\linewidth]{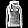}}
        \put(30, 2.5){\includegraphics[width=0.1\linewidth]{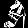}}
        
        \put(48, 14){\includegraphics[width=0.1\linewidth]{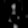}}
        \put(60, 14){\includegraphics[width=0.1\linewidth]{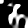}}
        \put(48, 2.5){\includegraphics[width=0.1\linewidth]{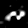}}
        \put(60, 2.5){\includegraphics[width=0.1\linewidth]{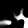}}
    
        \put(78, 14){\includegraphics[width=0.1\linewidth]{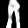}}
        \put(90, 14){\includegraphics[width=0.1\linewidth]{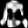}}
        \put(78, 2.5){\includegraphics[width=0.1\linewidth]{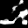}}
        \put(90, 2.5){\includegraphics[width=0.1\linewidth]{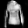}}

        \put(0, 12){\intab[0.85]{F-MNIST}}

        \put(16.5, 54){\intab[0.85]{Private Examples}}
        \put(47.5, 54){\intab[0.85]{InvertGrad~\cite{geiping2020inverting}}}
        \put(75.5, 54){\intab[0.85]{Proposed Attack}}

    \end{overpic}
    \caption{
    Comparison of the reconstructed images using the proposed attack method with InvertGrad~\cite{geiping2020inverting}. 
    The proposed method consistently yields more recognizable shapes and features. 
    \label{fig:recon_examples}
    }
\end{figure}
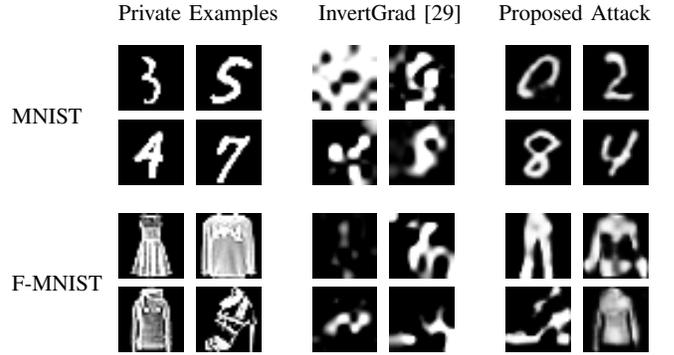
}

Finally, we explore the role of initial conditions in data reconstruction, as we have commented in Remark~\ref{remark:initial_cond}. 
From Section~\ref{sec:impact_progress}--\ref{sec:combined_defense}, we can observe that the increased task complexity from MNIST to F-MNIST results in a higher reconstruction error in all settings. 
This suggests that the base error term $e^{(0)}$ in Theorem~\ref{theorem:non_convex} is affected by the characteristics of the dataset and the complexity of the learning task. 
We further conduct case studies focusing on two factors: (i) initialization methods for model parameters and (ii) the choice of reconstruction algorithm. 
We begin by comparing three widely used initialization schemes, the Kaiming method~\cite{he2015delving}, the LeCun method~\cite{lecun2002efficient}, and the Orthogonal initialization~\cite{saxe2013exact}. 
TABLE~\ref{tab:reconstruction error_s4} reports the reconstruction error after 10 communication rounds under Passive Listener and Backdoor attack settings.  
Kaiming initialization consistently results in worse~(higher) reconstruction error than LeCun and Orthogonal initialization methods. 
This aligns with findings from previous work~\cite{he2015delving} that Kaiming initialization creates more favorable optimization landscapes for model training. 
From the attacker's perspective, this improved model conditioning leads to gradients that are harder to invert. 
These results indicate that the choice of model initialization influences the base error $e^{(0)}$ in Theorem~\ref{theorem:non_convex} and, therefore, privacy leakage. 

To evaluate the impact of the reconstruction function $\recon_{i}(\cdot)$, we compare the proposed reconstruction method with an adapted server-side attack method called \textit{inverting gradient attack~(InvertGrad)}~\cite{geiping2020inverting}. 
InvertGrad maximizes the cosine similarity between the observed and predicted gradients and incorporates a total variation prior to encouraging natural image structures. 
Figure~\ref{fig:recon_examples} shows qualitative comparisons between this baseline and our method under the Passive Listener setting at the first communication round of federated training. 
Quantitatively, the InvertGrad method gives mean RMSE values of $0.41$ and $0.43$ on MNIST and F-MNIST, respectively. 
In contrast, the proposed method yields mean RMSE values as low as $0.29$ and $0.33$. 
These results confirm our understanding of Remark~\ref{remark:initial_cond}, highlighting that the choice of reconstruction algorithm plays a crucial role in determining the final reconstruction quality.

\section{Conclusion}\label{sec:conclusion}

In this work, we have presented a novel threat model called the maliciously curious client. 
Our findings highlight a potential blind spot in the design of existing FL systems. 
Unlike traditional Byzantine clients that aim to corrupt training, the maliciously curious client seeks to extract sensitive information from peers by manipulating its own updates to amplify the gradient-based data reconstruction attack.
Through theoretical analysis and comprehensive experiments, we have shown that this attacker is able to achieve low reconstruction error, even when strong defenses such as differential privacy or robust aggregation are in place. 
Our findings demonstrate that privacy and robustness are deeply intertwined and some straightforward combinations of existing techniques do not necessarily yield stronger protection. 
In fact, in many scenarios, misconfigured defenses can amplify leakage beyond the ``no-protection'' baseline.
These results call for a rethinking of defense strategies in FL. 
Future work should aim to develop robust privacy-aware training frameworks that address both poisoning and the risk of data reconstruction from client-side attackers.


\IEEEpeerreviewmaketitle

\bibliographystyle{IEEEtran}
\bibliography{ref.bib}

\appendices

\setcounter{Lemma}{0}
\setcounter{Theorem}{0}

\section{ Proofs}\label{app:missing_proof}

We begin the proof of Theorem~\ref{theorem:non_convex} by giving a few observations. 

\begin{subequations}
\begin{align}
&\; \mathbb{E} \left\|\w[k+1]-\w[k]\right\|^2 \\
=&\; \mathbb{E}\left\| \frac{\eta}{M} \sum_{m=1}^{M} (\mathbf{n} + \tupdate[k,0]_{m}) \right\|^2  \\
=&\; \mathbb{E}\left\| \frac{\eta}{M} \sum_{m=1}^{M} (\mathbf{n} + \frac{1}{B}\sum_{i\in \mathcal{I}} \tupdate[k,0]_{m,i}) \right\|^2 \\
\overset{\cirone}{=} &\; \mathbb{E}\left\|\frac{\eta}{M} \sum_{m=1}^M  \mathbf{n} \right\|^2 + 
\mathbb{E}\left\| \eta   \frac{1}{MB} \sum_{m=1}^{M} \sum_{i\in \mathcal{I}} \tupdate[k,0]_{m,i} \right\|^2    \\ 
\overset{\cirtwo}{\leqslant} &\; \frac{\eta^2  d }{M B^2} C^2 \sigma^2 
+ \frac{\eta^2}{M^2B^2} \sum_{m=1}^M \sum_{i \in \mathcal{I}}  \mathbb{E}\| \tupdate[k,0]_{m,i} \|^2, \label{eq:interm0}
\end{align}
\end{subequations}
where $\cirone$ is based on the independence between noise and gradient, 
and $\cirtwo$ is from the i.i.d. sampling of the data on each client. 
Based on assumption~\ref{assumption:L_continuous}, we have
\begin{multline}\label{eq:L_smooth_result}
\mathbb{E} \left[f(\w[k+1]) - f(\w[k])\right]  \leqslant \\
- \eta \underbrace{\mathbb{E} \langle \nabla f(\w[k]), \update[k]\rangle }_{T_1} 
+ \frac{\eta^2 L_{g}}{2}  \expect \normsq{\update[k]}. 
\end{multline}
The term $T_1$ may be bounded as
\begin{subequations}
    \begin{align}
&\; \mathbb{E}\langle \nabla f(\w[k]) , \update[k] \rangle \\
= &\; \mathbb{E}\langle \frac{1}{M B} \sum_{m=1}^{M } \sum_{i \in \mathcal{I}} \update[k,0]_{m,i},  \frac{1}{M B} \sum_{m=1}^{M } \sum_{i \in \mathcal{I}} \tupdate[k,0]_{m,i} \rangle \\
\overset{\cirone}{=} &\; \frac{1}{M^2 B^2} \sum_{m=1}^M \sum_{i \in \mathcal{I}} 
\mathbb{E} \left[ \| \update[k,0]_{m,i}  \| \|\tupdate[k,0]_{m,i} \| \cos \theta \right] \\
\overset{\cirtwo}{=} &\; \frac{1}{M^2 B^2} \sum_{m=1}^M \sum_{i \in \mathcal{I}} 
\mathbb{E} \left[\| \update[k,0]_{m,i}  \| \|\tupdate[k,0]_{m,i} \| \right]\\ 
\overset{\cirthree}{\geqslant} &\; \frac{1}{M^2B^2} \sum_{m=1}^M \sum_{i \in \mathcal{I}} 
\mathbb{E}\| \tupdate[k,0]_{m,i} \|^2, \label{eq:inner_prod}
\end{align}
\end{subequations}
where $\cirone$ comes from i.i.d. sampling of the data on each client, with $\theta$ denoting the angle between $\update[k,0]_{m,i}$ and $\tupdate[k,0]_{m,i}$, 
$\cirtwo$ is based on the clipping operation given in \eqref{eq:clip_grad}, 
and $\cirthree$ is based on the following fact
\begin{equation}
     \left\{
    \begin{array}{ll}
        \| \update[k,0]_{m,i} \|  = \| \tupdate[k,0]_{m,i} \|, & \text{ if } \|\update[k,0]_{m,i}\| \leqslant C ,\\
        \| \update[k,0]_{m,i} \| > C \geqslant \| \tupdate[k,0]_{m,i} \|, & \text{ if } \|\update[k,0]_{m,i}\| > C. 
    \end{array}\right.
\end{equation}  
Plugging the results of \eqref{eq:interm0} and \eqref{eq:inner_prod} into \eqref{eq:L_smooth_result} yields 
\begin{multline}
\mathbb{E} \left[f(\w[k+1])\!-\!f(\w[k])\right] \leqslant \\ 
-\frac{\eta}{M^2B^2} (1 - \frac{\eta L_{g}}{2} ) \sum_{m=1}^M \sum_{i \in \mathcal{I}}  \mathbb{E}\| \tupdate[k,0]_{m,i} \|^2 + \frac{\eta^2 L_g d }{2 M B^2} C^2 \sigma^2.
\end{multline}
Let $\eta  = \frac{1}{L_g}$, we have 
\begin{multline}\label{eq:interm1}
    \sum_{m=1}^M \sum_{i \in \mathcal{I}} \mathbb{E}\| \tupdate[k,0]_{m,i} \|^2 \leqslant \\
    2 M^2B^2 \frac{ f(\w[k]) - f(\w[k+1]) }{\eta } 
    +  M d C^2 \sigma^2. 
\end{multline}
Substituting \eqref{eq:interm1} into \eqref{eq:interm0}, we obtain 
\begin{multline}\label{eq:weight_diff}
 \mathbb{E} \left\|\w[k+1]-\w[k]\right\|^2 \leqslant  \\ 
  \frac{2 \eta^2  d }{M B^2} C^2 \sigma^2  + 2\eta \left[ f(\w[k]) - f(\w[k+1]) \right]. 
\end{multline}
%
%
Now we move forward to prove Theorem~\ref{theorem:non_convex}. 

\begin{Theorem}
    Consider FedSGD with local DP against the client-side reconstruction attack. 
    Suppose that the conditions in Assumptions~\ref{assumption:L_continuous}--\ref{assumption:bounded_data} hold.  
    Let the learning rate $\eta = O\left(\frac{1}{L_{g}}\right)$, 
    then at the $k$th communication round, the expected reconstruction error satisfies the following inequality, 
\begin{equation}
    \textrm{\normalfont RMSE}(\hX^{(k)}) \! \leqslant \!\min\!\left\{  O\!\left(\rho_0^{(k)}\!\Delta^{(k+1)}\! + \! \rho_1^{(k)}\!\sigma \!+\! e^{(0)}\right) \!,\!  \frac{2\upsilon}{\sqrt{d_{\text{in}}}} \right\},
\end{equation}
where $\Delta^{(k+1)} = [ f(\w[0]) - f(\w[k+1])]^{1/2}$ is the objective gap, 
$\rho_0^{(k)} = 2 L_{\psi} \sqrt{2 L_g k}$ and 
$\rho_1^{(k)} = \frac{2 \sqrt{2d} L_{\psi}  C k }{\sqrt{M} B}$ are the time-variant terms,
and $e^{(0)} = \max_{i} \| \x_i - \recon_{i}(\update[0])  \| $ is the base error term. 
\end{Theorem}

\begin{proof}
We have 
\begin{subequations}
\begin{align}
    \text{RMSE}(\hX^{(k)}) &= \sqrt{ \frac{1}{d_{\text{in}} N} \mathbb{E} \| \Recon(\update) -  \X\|^2_{\text{F}} } \\ 
    & =  \frac{1}{\sqrt{ d_{\text{in}} N}} \sqrt{\sum_{i} \mathbb{E} \| \recon_{i}(\update) -  \x_i\|^2_{2} } \\
    &\leqslant \frac{1}{\sqrt{ d_{\text{in}} }  }  \max_{i}  \mathbb{E}\underbrace{ \| \recon_{i}(\update[k]) - \x_{i}  \|_2}_{T_2}. \label{eq:first_ineq}
\end{align}
\end{subequations}
The term $T_2$ may be bounded as 
\begin{subequations}
\begin{align}
&\;    \| \x_{i} - \recon_{i}(\update[k]) \| \\
 = & \; \| \x_{i} - \recon_{i}(\update[0]) + \recon_{i}(\update[0]) - \recon_{i}(\update[k]) \| \\ 
\overset{\cirone}{\leqslant} &\; \underbrace{\| \x_{i} - \recon_{i}(\update[0])\|}_{\triangleq e_{i}^{(0)}} + \| \recon_{i}(\update[0]) - \recon_{i}(\update[k]) \| \\
\overset{\cirtwo}{\leqslant} &\; e_{i}^{(0)} + L_{\psi} \underbrace{\| \update[0] - \update[k]  \|}_{T_3}, \label{eq:recon_err0}
\end{align}
\end{subequations}
where $\cirone$ holds due to triangle inequality and $\cirtwo$ comes from Assumption~\ref{assumption:L_continuous}. 
To bound term $T_3$, we first inspect 
\begin{subequations}
\begin{align}
  & \;  \mathbb{E} \| \update[0] - \update[1]  \|^2 \\
= & \; \frac{1}{\eta^2} \mathbb{E}\|  \w[0] - \w[1] - ( \w[1]  - \w[2] )\|^2 \\
\overset{\cirone}{\leqslant} & \; \frac{2}{\eta^2} \left[  \mathbb{E}\|\w[0]  - \w[1]\|^2  + \mathbb{E}\| \w[1] - \w[2] \|^2 \right] \\
\overset{\cirtwo}{\leqslant} & \; \frac{2}{\eta^2} \Bigg[ 2\eta \left(f(\w[0]) - f(\w[1])\right) +  \frac{2 \eta^2  d }{M B^2} C^2 \sigma^2  \\
& \; \; + 2 \eta \left(f(\w[1]) - f(\w[2])\right) +  \frac{2 \eta^2  d }{M B^2} C^2 \sigma^2   \Bigg] \\
\leqslant &\; \frac{4}{\eta} \left[ f(\w[0]) - f(\w[2]) +  \frac{2 \eta  d }{M B^2} C^2 \sigma^2 \right], \label{eq:g_diff}
\end{align}
\end{subequations}
where $\cirone$ is based on $ \| \mathbf{a} + \mathbf{b} \|^2 \leqslant 2(\| \mathbf{a}\|^2 + \|\mathbf{b} \|^2)$ and $\cirtwo$ is from \eqref{eq:weight_diff}. 
Substituting \eqref{eq:g_diff} into the squared form of $T_3$ yields  
\begin{subequations}
\begin{align}
& \;    \mathbb{E}\| \update[0] - \update[k]  \|^2 \\
= & \;  \mathbb{E}\|\sum_{u=0}^{k-1} (\update[u] - \update[u+1]) \|^2 \\
\overset{\cirone}{\leqslant} &\; k  \sum_{u=0}^{k-1} \mathbb{E}\| \update[u] - \update[u+1]   \|^2 \\
\overset{\cirtwo}{\leqslant} &\; \frac{4k}{\eta} \big[ f(\w[0]) + f(\w[1]) - f(\w[k]) - f(\w[k+1]) \\
&\; \; + \frac{2 k \eta  d }{M B^2} C^2 \sigma^2  \big] \\
\overset{\cirthree}{\leqslant} &\; \frac{8k}{\eta} \left[ f(\w[0]) - f(\w[k+1])  +  \frac{ k \eta  d }{M B^2} C^2 \sigma^2 \right], \label{eq:interm2}
\end{align}
\end{subequations}
where $\cirone$ is based on Jensen's inequality on squared norm, $\cirtwo$ is from \eqref{eq:g_diff} and $\cirthree$ is based on the descent lemma such that $f(\w[1])  \leqslant f(\w[0])$ and $f(\w[k+1]) \leqslant f(\w[k])$. 
Substituting \eqref{eq:interm2} into \eqref{eq:recon_err0} yields
\begin{subequations}
\begin{align}
&\;  \mathbb{E} \| \x_i - \recon_{i}(\update[k]) \| \\ \leqslant
&\;  e_{i}^{(0)} + \frac{2 L_{\psi} \sqrt{2k}}{\sqrt{\eta}}[ f(\w[0])  - f(\w[k+1]) + \frac{ k \eta  d }{M B^2} C^2 \sigma^2 ]^{\frac{1}{2}} \\
\overset{\cirone}{\leqslant} &\; e_{i}^{(0)} + \underbrace{2 L_{\psi} \sqrt{\frac{2k}{\eta}}}_{\triangleq \rho_0^{(k)}}  \Delta^{(k+1)} + \underbrace{\frac{2 \sqrt{2d} L_{\psi} k }{\sqrt{M} B} C}_{\triangleq \rho_1^{(k)}} \sigma,   
\end{align}
\end{subequations}
where $\cirone$ comes from subadditivity of the square root function. 
Here, $\Delta^{(k+1)} = [f(\w[0]) - f(\w[k+1])]^{1/2}$. 
Meanwhile, based on Assumption~\ref{assumption:bounded_data}, the reconstruction error is upper bounded by 
\begin{subequations}
\begin{align}
    \| \x - \recon_{i}(\update[k]) \| \leqslant \| \x \| + \| \recon_{i}(\update[k]) \| = 2 \upsilon. 
\end{align}
\end{subequations} 
Therefore, we have 
\begin{equation}
    \mathbb{E}\| \x - \recon_{i}(\update[k]) \| \leqslant \min\{ e_{i}^{(0)} \!+\! \rho_0^{(k)} \Delta^{(k+1)}\! + \! \rho_1^{(k)} \sigma, 2\upsilon \}.
\end{equation}
Applying the results to \eqref{eq:first_ineq}, we derive the following results
\begin{subequations}
\begin{align}
 &\;   \sqrt{d_{\text{in}}} \operatorname{RMSE}(\X^{(k)}) \\
 \leqslant &\; \min\{ e^{(0)} \!+\! \rho_0^{(k)} \Delta^{(k+1)}\! + \! \rho_1^{(k)} \sigma, 2\upsilon \},
\end{align}
\end{subequations}
where $e^{(0)} = \max_{i} \| \x_i - \recon_{i}(\update[0])  \| $. 
The choice of the learning rate $\eta = O(\frac{1}{L_g})$ gives the final result in $O(\cdot)$ form. 
 
\end{proof}

\section{Byzantine Attacks and Robust Aggregations}\label{app:robust_agg}
We review some state-of-the-art Byzantine robust defenses in the literature. 
The input $\update[k]_{\mathcal{C}}$ denotes the set of client gradients for aggregation. 

\highlight{Balance~\cite{fang2024byzantine}.}
Balance is an uninformed defense that relies on an evaluation dataset for reference. 
Suppose that the defender calculates a reference update $\update_{\text{ref}}$ based on the dataset $\mathcal{D}_{\text{ref}}$.
For each received gradient $\update[k,\tau]_{m}$, the defender chooses to accept it if the following inequality holds, 
\begin{equation}\label{eq:balance}
    \| \update[k]_{\text{ref}} - \update[k,\tau]_{m} \| \leqslant \phi \exp\left[- \kappa \lambda(k) \right] \| \update[k]_{\text{ref}} \|,  
\end{equation}
where $\phi>0$ is a scaling factor, $\kappa>0$ controls how fast the exponential function decays, and $\lambda(k) = k/K$ is a monotonic increasing function with respect to communication round index $k$. 
After a benign set $\mathcal{S}_{\text{B}}$ is selected based on $\eqref{eq:balance}$, the aggregation results can be obtained via FedAvg, 
\begin{equation}
    \widehat{\Phi}_{\text{Balance}}(\update[k]_{\mathcal{C}}; \mathcal{D}_{\text{ref}}) = \frac{1}{|\mathcal{S}_{\text{B}}|}\sum_{b \in \mathcal{S}_{\text{B}}} \update[k,\tau]_{b}. 
\end{equation}

\highlight{Frequency Analysis-Based Method (FreqFed)}~\cite{fereidooni2024freqfed}. 
The defender first applies the discrete cosine transform to the received gradients and obtains transformed coefficients. 
The coefficients are then processed with lowpass filtering to reduce the influence of high-frequency noise. 
Finally, hierarchical density-based spatial clustering (HDBSCAN) is applied to the pairwise cosine distance of the low-frequency coefficients. 
The largest cluster is selected as the benign cluster $\mathcal{S}_{\text{F}}$. 
An aggregated update is obtained by performing FedAvg on the selected cluster: 
\begin{equation}
    \widehat{\Phi}_{\text{FreqFed}}(\update[k]_{\mathcal{C}}) = \frac{1}{|\mathcal{S}_{\text{F}}|}\sum_{b \in \mathcal{S}_{\text{F}}} \update[k,\tau]_{b}. 
\end{equation}

\highlight{SignGuard~\cite{xu2022byzantine}.}
SignGuard utilizes the statistics of gradient signs to filter out malicious updates. 
The defender initially creates a set of indices, $\mathcal{S}_1$, by identifying and excluding outlier gradients. 
Concurrently, another set of indices, $\mathcal{S}_2$, is formed by selecting the largest cluster based on sign agreement. 
The updates are then aggregated using client indices from the intersection $\mathcal{S}_{\text{sg}} = \mathcal{S}_1 \cap \mathcal{S}_2$:
\begin{equation}
    \widehat{\Phi}_{\text{SignGuard}}(\update[k]_{\mathcal{C}}) = \frac{1}{|\mathcal{S}_{\text{sg}} |}\sum_{b \in \mathcal{S}_{\text{sg}}} \update[k,\tau]_{b}. 
\end{equation}

\highlight{Krum and Multi-Krum~\cite{blanchard2017machine}.}
Krum defender chooses one client update that is closest to its neighbors, according to the following operation:
\begin{equation}\label{eq:krum}
    \widehat{\Phi}_{\text{Krum}}(\update[k]_{\mathcal{C}}; A) = \underset{\update[k,\tau]_{i}}{\operatorname{argmin}} \sum_{i \rightarrow j}\left\|\update[k,\tau]_i-\update[k,\tau]_j\right\|^2,
\end{equation}
where $i \rightarrow j$ is the indices of the $M-A-2$ nearest neighbors of $\update[k,\tau]_{i}$ based on the Euclidean distance.
Multi-Krum extends Krum by selecting $c$ model updates and averages selected updates. 
Specifically, Multi-Krum performs Krum in~\eqref{eq:krum} $c$ times, each time selecting an update and moving it from the received update set $\update[k]_{\mathcal{C}}$ to the Multi-Krum candidate set.

\highlight{Divide and Conquer (DnC)~\cite{shejwalkar2021manipulating}.}
DnC is an informed defense that assumes knowledge of the number of attackers, $A$.
DnC defender first randomly selects a set of indices of coordinates to sparsify/subsample the gradients, keeping $s$ valid entries out of $d$.
The defender then constructs an $M \times s$ matrix $\mG$ by concatenating subsampled gradients and normalizing them to $\tilde{\mG}$ by subtracting the mean gradients. 
DnC detects the attackers by projecting these centered gradients in $\tilde{\mG}$ along their principal right singular eigenvector $\bm{v}$ and determining outlier scores. 
Given a filtering factor $c \in (0,1)$, gradients with the lowest outlier scores $M - c A$ are selected as benign updates.
Thus, a benign index set $\mathcal{S}_{\text{D}}$ can be constructed.  
The final update is computed by averaging the selected updates:
\begin{equation}
    \widehat{\Phi}_{\text{DnC}}(\update[k]_{\mathcal{C}}; A) = \frac{1}{|\mathcal{S}_{\text{D}}|}\sum_{b \in \mathcal{S}_{\text{D}}} \update[k,\tau]_{b}. 
\end{equation}

\section{Experimental Setup}\label{app:setup}

We conduct experiments on the Fashion MNIST (F-MNIST)~\cite{xiao2017fashion} and MNIST~\cite{deng2012mnist} datasets. 
F-MNIST is a collection of $28 \times 28$ pixels of grayscale images spanning 10 categories of clothing items, containing $60{,}000$ training and $10{,}000$ test samples. 
MNIST is a dataset with a similar data structure containing 10  categories of handwritten digits. 
For this dataset, we use the LeNet-5 architecture~\cite{el2017cnn}, a classical convolutional neural network composed of two convolutional layers and three fully connected layers. 
The neural network model is initialized with the PyTorch default implementation, i.e., Kaiming Uniform method. 
To simulate non-IID data in FL, we use the Dirichlet-based partitioning method proposed by Hsu et al.~\cite{hsu2019measuring}, 
$\bm{q}_m \sim \text{Dir}(\alpha)$, where $\bm{q}_m = [q_{m,1}, \dots, q_{m,C}]^{\top}$ belongs to the $(C-1)$-standard simplex.  
The client's local dataset is then populated such that category $k$ appears in proportion to $(100 \cdot q_{m,k})\%$ of the samples. 
We use $\alpha=0.1$ by default. 
We simulate an FL setup with $M = 10$ clients and full participation in each round. 
Unless otherwise specified, each client trains locally for $5$ steps, with a batch size of $B=100$ and a learning rate of $\eta=0.01$. 
For the reconstruction optimization, we use Adam optimizer and with learning rate $\eta$ selected from $\{10^{-4}, 3\times 10^{-3}, 10^{-3}, 3\times 10^{-3}, 10^{-2} \}$ via validation. 
We use a three-layer convolutional neural network for the autoencoder $E(\cdot)$. 
For the local DP method, we set clipping bound to $C=4$. 
The poisoning functions adopt the default parameters in the respective papers.  
The mean value of experimental results are reported over five independent runs.

\end{document}